%% file: main.tex
\documentclass{article}

\usepackage{booktabs}       
\usepackage{amsfonts}       
\usepackage{nicefrac}       
\usepackage{microtype}      
\usepackage{xcolor}         
\usepackage[ruled,noend,linesnumbered,algo2e]{algorithm2e}  
\usepackage{graphicx}	
\usepackage{array}
\usepackage{float}
\input{preamble}

\newcolumntype{T}[1]{%
    >{\centering\arraybackslash\hspace{0pt}}p{#1}}%

\title{New Coresets for Projective Clustering and
Applications}
\author{
\hspace{0.9in} Murad Tukan
\thanks{University of Haifa.  
E-mail: \email{muradtuk@gmail.com}}\\
\and
Xuan Wu\thanks{Johns Hopkins University. 
E-mail: \email{wu3412790@gmail.com}}
\and
Samson Zhou\thanks{Carnegie Mellon University. 
E-mail: \email{samsonzhou@gmail.com}}\hspace{0.9in}
\and
Vladimir Braverman\thanks{Johns Hopkins University. 
E-mail: \email{vova@cs.jhu.edu}}
\and
Dan Feldman\thanks{University of Haifa. 
E-mail: \email{dannyf.post@gmail.com}}
}

\date{\today}

\begin{document}
\maketitle

\begin{abstract}
$(j,k)$-projective clustering is the natural generalization of the family of $k$-clustering and $j$-subspace clustering problems. Given a set of points $P$ in $\mathbb{R}^d$, the goal is to find $k$ flats of dimension $j$, i.e., affine subspaces, that best fit $P$ under a given distance measure. In this paper, we propose the first algorithm that returns an $L_\infty$ coreset of size polynomial in $d$. Moreover, we give the first strong coreset construction for general $M$-estimator regression. Specifically, we show that our construction provides efficient coreset constructions for Cauchy, Welsch, Huber, Geman-McClure, Tukey, $L_1-L_2$, and Fair regression, as well as general concave and power-bounded loss functions. Finally, we provide experimental results based on real-world datasets, showing the efficacy of our approach.
\end{abstract}

\section{INTRODUCTION}
Coresets are often used in machine learning, data sciences, and statistics as a pre-processing dimensionality reduction technique to represent a large dataset with a significantly smaller amount of memory, thereby improving the efficiency of downstream algorithms in both running time and working space.
Intuitively, a coreset $C$ of a set $P$ of $n$ points in $\mathbb{R}^d$ is a smaller number of weighted representatives of $P$ that can be used to approximate the cost of any query from a set of a given queries.
Hence rather than optimizing some predetermined objective on $P$, it suffices to optimize the objective on $C$, which has significantly smaller dimension than $P$.
In this paper, we present coresets for projective clustering.

Projective clustering is an important family of clustering problems for applications in unsupervised learning~\citep{Procopiuc10}, data mining~\citep{AggarwalPWYP99,AggarwalY00}, computational biology~\citep{Procopiuc10}, database management~\citep{ChakrabartiM00}, and computer vision~\citep{ProcopiucJAM02}.
Given a set $P$ of $n$ points in $\mathbb{R}^d$, a parameter $z$ for the exponent of the distance, and a parameter $k$ for the number of flats of dimension $j$, the $(j,k)$-projective clustering problem is to find a set $\calF$ of $k$ $j$-flats that minimizes the sum of the distances of $P$ from $\calF$, i.e., $\min_{\calF}\sum_{\p\in P}\dist(\p,\calF)^z$, where $\dist(\p,\calF)^z$ denotes the $z$-th power of the Euclidean distance from $\p$ to the closest point in any flat in $\calF$.
We abuse notation by defining the projective clustering problem to be $\min_{\calF}\max_{\p\in P}\dist(\p,\calF)$ for $z=\infty$.
Projective clustering includes many well-studied problems such as the $k$-median clustering problem for $z=1$, $j=0$, $k\in\mathbb{Z}^+$, the $k$-means clustering problem for $z=2$, $j=0$, $k\in\mathbb{Z}^+$, the $k$-line clustering problem for $z\ge 0$, $j=1$, $k\in\mathbb{Z}^+$, the subspace approximation problem for $z\ge 0$, $j\in\mathbb{Z}^+$, $k=1$, the minimum enclosing ball problem for $z=\infty$, $j=0$, $k=1$, the $k$-center clustering problem for $z=\infty$, $j=0$, $k\in\mathbb{Z}^+$, the minimum enclosing cylinder problem for $z=\infty$, $j=1$, $k=1$, and the $k$-cylinder problem for $z=\infty$, $j=1$, $k\in\mathbb{Z}^+$.

\subsection{Related Work}
Finding the optimal set $C$ for projective clustering is known to be NP-hard~\cite{AloiseDHP09} and even finding a set with objective value that is within a factor of $1.0013$ of the optimal value is NP-hard~\cite{LeeSW17}.
\cite{ProcopiucJAM02} implemented a heuristics-based Monte Carlo algorithm for projective clustering while~\cite{Har-PeledV02} introduced a dimensionality reduction technique to decrease the size of each input point, which distorts the cost of the optimal projective clustering. 
Similarly, \cite{KerberR15} used random projections to embed the input points into a lower dimensional space. 
However, none of these approaches reduces the overall number of input points, whose often causes the main bottleneck for implementing approximation algorithms for projective clustering in big data applications.

\cite{BadoiuHI02} first introduced coresets for the $k$-center and $k$-median clustering problems in Euclidean space.
Their coresets constructions gave $(1+\eps)$-approximations and sampled a number of points with exponential dependency in both $\frac{1}{\eps}$ and $k$.
Their work also inspired a number of coresets for specific projective clustering problems; coresets have subsequently been extensively studied in $k$-median or $k$-means clustering~\citep{BadoiuHI02,Har-PeledM04,FrahlingS05,FrahlingS08,Chen09,FeldmanS12,BravermanLUZ19,HuangV20}, subspace approximation~\citep{DeshpandeRVW06,DeshpandeV07,FeldmanL11,FeldmanMSW10,ClarksonW15,SohlerW18,FeldmanSS20, tukan2021no}, and a number of other geometric problems and applications~\citep{AgarwalHY06,FeldmanFS06,Clarkson08,DasguptaDHKM08,
AckermannB09,PhillipsT18,HuangJLW18,AssadiBBMS19,MunteanuSSW18,
BravermanDMMUWZ20,MussayOBZF20,maalouf2020tight,tukan2020coresets,tukan2021coresets,jubran2020sets,maalouf2021coresets}.
However, these coreset constructions were catered toward specific problems rather than the general $(j,k)$-projective clustering problem.

\cite{FeldmanL11} introduced a framework for constructing coresets by sampling each input point with probability proportional to its \emph{sensitivity}, which informally quantifies the importance of the point with respect to the predetermined objective function. 
\cite{FeldmanL11} also performed dimensionality reduction for $(j,k)$-projective clustering by taking the union of two sets $\mathcal{S}$ and $\textrm{proj}(P,B)$, where $P$ is the input data set of size $n$. 
Although the set $\mathcal{S}$ can have size $\poly(j,k,d)$, the set $\textrm{proj}(P,B)$ still has size $n$, so their resulting output can actually have \emph{larger} size than the original input. 
The main point is that $\textrm{proj}(P,B)$ lies in a low-dimensional space, so their approach should be viewed as a dimensionality reduction technique to decrease the ambient dimension $d$ whereas our coreset construction decreases the input size $n$. \citep{ClarksonW15} suggested approximation algorithms based on matrix sketches for $(1,j)$-projective clustering problems with respect to family of $M$-estimator functions, and~\citep{clarkson2019dimensionality} provided tighter result with respect to the $(1,j)$-projective clustering problems with respect to the Tukey loss function.
\citep{VaradarajanX12} proved upper bounds for the total sensitivity of the input points for a number of shape fitting problems, including the $k$-median, $k$-means, and $k$-line clustering problems, as well as an $L_1$ coreset for the integer $(j,k)$-projective clustering problem.
On the other hand, \citep{Har-Peled04} showed that $L_\infty$ coresets for the projective clustering problem does not exist even for $j=k=2$ when the input set consists of points from $\mathbb{R}^d$.
When the input is restricted to integer coordinates, \citep{EdwardsV05} constructed an $L_\infty$ coreset that gives a $(1+\eps)$-approximation for $(j,k)$-projective clustering.
However, their construction uses a subset of points with size exponential in both $k$ and $d$, which often prevents practical implementations.
Hence, a natural open question is whether there exist $L_\infty$ coreset constructions for integer $(j,k)$-projective clustering with size polynomial in $d$.

\subsection{Our Contributions}
We give the first $L_\infty$ coreset construction for the integer $(j,k)$-projective clustering problem with size polynomial in $d$, resolving the natural open question from \cite{EdwardsV05}.
Specifically, we give an $L_\infty$ $\xi$-coreset $C$, so that for any choice $\calF$ of $k$ flats with dimension $j$, the maximum connection cost of $C$ to $\calF$ is at most $\xi$ times the maximum connection cost of $P$.
Previously, even in the case of $k=1$ and constant $j$, the best known $L_\infty$ coreset construction had size $\exp(d)$~\citep{EdwardsV05}.
We first introduce an $L_\infty$ coreset construction for the $(j,1)$-projective clustering problem using Carath\'{e}odory's theorem; see Figure~\ref{fig:illustration}. We then use our $L_\infty$ coreset for $(j,1)$-projective clustering as a base case to recursively build a coreset $D_k$ for $(j,k)$-projective clustering from coresets for $(j,k-1)$-projective clustering on the partitions of the input points that have geometrically increasing distances from the affine subspace spanned by the points chosen in the previous steps. 
We use properties from \cite{EdwardsV05,FeldmanSS20} to bound the number of partitions determined by the distances from the input points to each of the affine subspaces, which bounds our coreset size for an input with aspect ratio $\Delta$, i.e., the ratio of the largest and smallest coordinate magnitudes. 

\begin{theorem}[Small $L_\infty$ coreset for $(j,k)$-projective clustering]
\label{thm:main:infty}
There exists an $L_\infty$ constant-factor approximation coreset for the $(j,k)$-projective clustering problem with size $(8j^3\log(d\Delta))^{\O{jk}}$.
\end{theorem}

Our main technical contribution is the novel $L_\infty$ coreset construction for the $(j,1)$-projective clustering problem that relies on Carath\'{e}odory's  theorem, which we crucially use to form the base case in our recursive argument. 
We then build upon our novel coreset construction by adding a polynomial number of points to the coreset over each step in the inductive argument. 
By comparison, even the base case for the previous best coreset~\citep{EdwardsV05} uses exponential space by essentially constructing an epsilon net with $\left(\frac{1}{\epsilon}\right)^{\O{d}}$ points. 

We then give the first $L_\infty$ coresets for a number of $M$-estimator regression problems. 
Although the framework of Theorem~\ref{thm:main:infty} immediately gives coreset constructions for Cauchy, Welsch, Huber, Geman-McClure, Tukey, $L_1-L_2$, and Fair regression, we instead apply sharper versions of the proof of Theorem~\ref{thm:main:infty} to the respective parameters induced by each of the loss functions to obtain even more efficient coreset constructions. 
Our constructions give strong coresets so that with high probability, the data structure simultaneously succeeds for all queries. 
We then apply the framework of Theorem~\ref{thm:main:infty} to give $L_\infty$ coresets for any non-decreasing concave loss function $\Psi$ with $\Psi(0)=0$. 
We generalize this approach to give $L_\infty$ coresets for any non-decreasing concave loss function $\Psi$ with $\Psi(y)/\Psi(x)\le(y/x)^z$ for a fixed constant $z>0$, for all $0\le x\le y$. 
Note that this property essentially states that the loss function $\Psi(x)$ is bounded by some power function $x^z$. 
We summarize these results in Table~\ref{table:Mestimators}. 

We also use Theorem~\ref{thm:main:infty} along with the well-known sensitivity sampling technique to obtain an $L_2$ coreset for integer $(j,k)$-projective clustering with approximation $(1+\eps)$. 

\begin{figure*}[!htb]
\centering
    \begin{subfigure}[t]{0.20\textwidth}
		\centering
		\includegraphics[width = 0.8\textwidth]{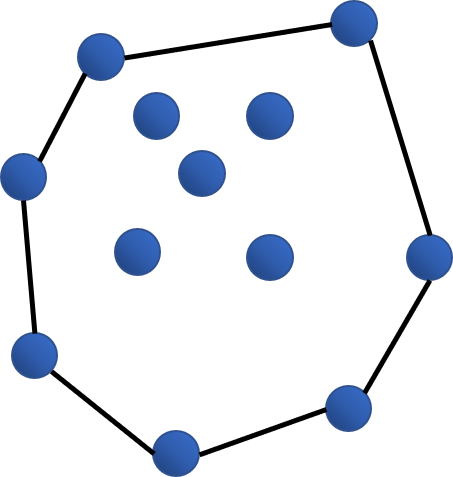}
        \label{fig:step1}
    \end{subfigure}
    \begin{subfigure}[t]{0.20\textwidth}
		\centering
		\includegraphics[width = 0.8\textwidth]{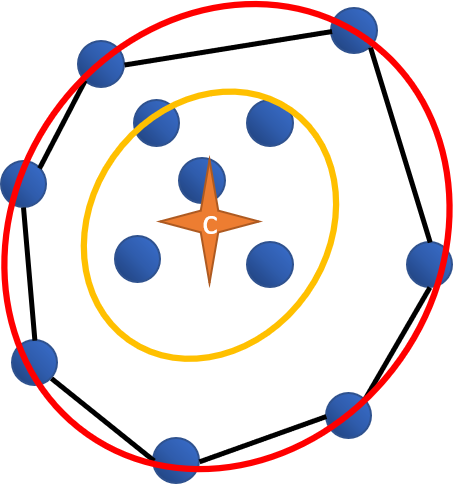}
        \label{fig:step2}
    \end{subfigure}
	\begin{subfigure}[t]{0.20\textwidth}
		\centering
		\includegraphics[width = 0.8\textwidth]{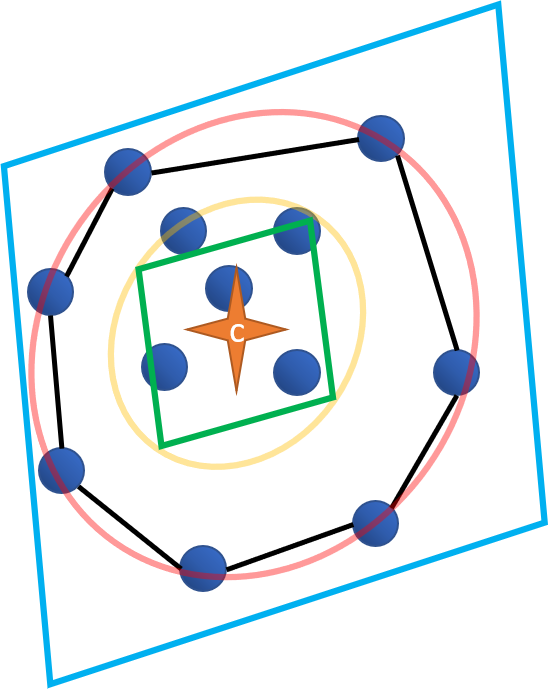}
        \label{fig:step3}
	\end{subfigure}
    \begin{subfigure}[t]{0.20\textwidth}
		\centering
		\includegraphics[width = 0.8\textwidth]{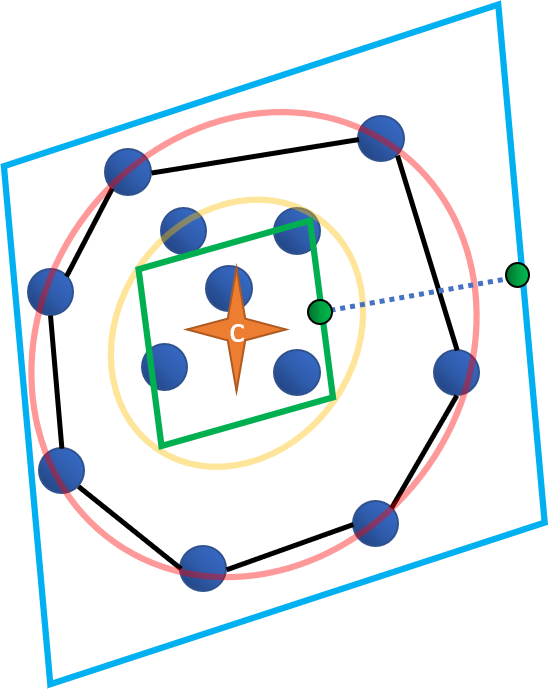}
        \label{fig:step4}
    \end{subfigure}
    \caption{\textbf{Overview of our approach (see Algorithm~\ref{alg:single:projective}).} Images from left to right: Steps $1$ and $2$: An approximated ellipsoid $E$ is computed satisfying the conditions of Theorem~\ref{thm:loewner:ellipsoid} (red ellipsoid), where the ellipsoid $E^\prime$ (in orange) is a dilation of $E$ by a factor of $\alpha=d$ with respect to the center $c$ of $E$ (orange star). Step $3$: The vertices of $E^\prime$ (orange ellipsoid) are computed and then dilated such that their convex hull (cyan outline) will contain $\conv(P)$ (black outline). Step $4$: A Caratheodory set of $d+1$ points from $P$ is computed for each vertex point of $E^\prime$ (green convex hull). Finally, each point on $\conv(P)$ (black outline) can be represented by a convex combination of a point on the convex hull of the vertices of $E^\prime$ and their dilated points (cyan outline).}
\label{fig:illustration}
\end{figure*}

\begin{theorem}[Small $L_2$ coreset for $(j,k)$-projective clustering]
\label{thm:main:two}
There exists an $L_2$ coreset with approximation guarantee $(1+\eps)$ for the $(j,k)$-projective clustering problem with size $\O{(8j^3\log(d\Delta))^{\O{jk}}\log n}$.
\end{theorem}
\begin{table*}[!htb]
\centering
\caption{$M$-estimator loss functions that can be captured by our coreset construction; $d$ here denotes the dimension of the input data $P$; all lemmata below can be found at Section~\ref{supplement:app} of the supplementary material.}
\begin{tabular}{l|T{0.35\textwidth}|T{0.2\textwidth}|r}
\hline
Loss Function $\Psi$ & Formulation & multiplicative error ($\ell_\infty$-coreset) & Reference \\
\hline
Cauchy & $\left(\lambda^2/2\right) \log{\left( 1 + (x/\lambda)^2\right)}$ & $8(d+1)^3$ & Lemma~\ref{lem:cauchy} \\ \hline
Welsch & $\frac{\lambda^2}{2}\left(1-e^{-\left(\frac{x}{\lambda}\right)^2}\right)$ & $8(d+1)^3$ & Lemma~\ref{lem:welsch} \\ \hline
Huber     &  $\begin{cases} x^2/2 & \text{If } \abs{x} \leq \lambda\\
 \lambda\abs{x} - \lambda^2/2 & \text{otherwise} \end{cases}$ & $16(d+1)^3$ & Lemma~\ref{lem:huber} \\ \hline
Geman-McClure & $x^2/\left(2 + 2x^2\right)$ & $8(d+1)^3$ & Lemma~\ref{lem:gm} \\ \hline
Concave & $\frac{d^2\Psi}{dx^2}\le 0$ & $4(d+1)^{1.5}$& Lemma~\ref{lem:concave} \\ \hline
Tukey & $\begin{cases} \frac{\lambda^2}{6}\left(1-\left(1-\frac{x^2}{\lambda^2}\right)^3\right) & \text{if } \abs{x} \leq \lambda\\
 \frac{\lambda^2}{6} & \text{otherwise} \end{cases}$ & $8(d+1)^3$ & Lemma~\ref{lem:tukey} \\ \hline
$L_1-L_2$ & $2\left(\sqrt{1+x^2/2}-1\right)$ & $8(d+1)^3$ & Lemma~\ref{lem:ll} \\ \hline
Fair & $\lambda|x|-\lambda^2\ln\left(1+|x|/\lambda\right)$  & $8(d+1)^3$ & Lemma~\ref{lem:fair} \\ \hline
Power Bounded & $\Psi_{Pow}(y)/\Psi_{Pow}(x)\le(y/x)^z$ for all $0\le x\le y$ & $4^z(d+1)^{1.5z}$ & Lemma~\ref{lem:power} \\ 
\hline
\end{tabular}
\label{table:Mestimators}
\end{table*}


\textbf{Experiments.}
Finally, we complement our theoretical results with empirical evaluations on synthetic and real world datasets for regression and clustering problems. 
We first consider projective clustering on a bike sharing dataset and a 3D spatial network from the UCI machine learning repository~\citep{Dua:2019}. 
We then generate a synthetic dataset in the two-dimensional Euclidean plane. 
Since previous coreset constructions with theoretical guarantees are impractical for implementations, we compare our algorithms to a baseline produced by uniform sampling. 
Our experiments demonstrate that our algorithms have superior performance both across various ranges of $j$ and $k$ for the $(j,k)$-projective clustering problem as well as across various regression problems, e.g., Cauchy, Huber loss functions. 

\subsection{Preliminaries}
For a positive integer $n$, we write $[n]:=\{1,\ldots,n\}$.
We use bold font variables to denote vectors and matrices.
For a vector $\x\in\mathbb{R}^d$, we have the Euclidean norm $\|\x\|_2=\sqrt{\sum_{i=1}^d x_i^2}$.
We use $\log$ to denote the base two logarithm.
We use the notation $\circ$ to denote vertical concatenation, so that if $\u$ and $\v$ are row vectors with dimension $d$, then $\u\circ\v$ is the matrix with dimension $2\times d$ whose first row is $\u$ and second row is $\v$. Recall that for $\c\in\mathbb{R}^d$ and a symmetric positive definite matrix $\G\in\mathbb{R}^{d\times d}$, we define the ellipsoid $E(\G,\c)$ to be the set $E(\G,\c):=\br{\x\in\mathbb{R}^d\,|\,(\x-\c)^\top\G(\x-\c)\le1}.$

\begin{theorem}[John-L\"{o}wner ellipsoid]
\citep{John14}
\label{thm:loewner:ellipsoid}
For a set $L\subseteq\mathbb{R}^d$ of points with nonempty interior, there exists an ellipsoid $E(\G,\c)$, where $\G\in\mathbb{R}^{d\times d}$ is a positive definite matrix and $\c\in\mathbb{R}^d$, of minimal volume such that $\frac{1}{d}(E(\G,\c)-\c)+\c\subseteq\conv(L)\subseteq E(\G,\c).$
\end{theorem}

The following defines an approximated solution to problem of finding the L\"{o}wner ellipsoid.
\begin{definition}[$\alpha$-rounding]
\citep{todd2007khachiyan}
Let $L\subseteq\mathbb{R}^d$ be a finite set such that $\Span(L)=\mathbb{R}^d$ and let $\alpha\ge 1$.
Then an ellipsoid $E(\G,\c)$ is called an $\alpha$-rounding of $\conv(L)$ if $\frac{1}{\alpha}(E(\G,\c)-\c)+\c\subseteq\conv(L)\subseteq E(\G,\c).$
\end{definition}
Note that if $\alpha$ in the above definition is $d$ (or equiv. $\sqrt{d}$), the corresponding ellipsoid is the L\"{o}wner ellipsoid.

In order to define a distance to any affine subspace, we first need the following ingredients.
\begin{definition}[Orthogonal matrices]
Let $d>j\ge 1$ be integers.
We say $\X\in\mathbb{R}^{d\times j}$ is an orthogonal matrix if $\X^\top\X=\I_j$.
We use $\calV_j\subseteq\mathbb{R}^{d\times j}$ to denote the set of all $d\times j$ orthogonal matrices.
\end{definition}

\begin{definition}[$j$-dimensional subspace]
Let $d>j\ge 1$ be integers and let $\v\in\mathbb{R}^d$.
Let $\X\in\calV_j$ and $\Y\in\calV_{d-j}$ such that $\Y^\top\X=0^{(d-j)\times j}$ and $\X^\top\Y=0^{j\times(d-j)}$.
Let $H(\X,\v):=\{\X\X^\top\p+\v\,|\,\p\in\mathbb{R}^d\}$ denote the $j$-dimensional affine subspace $H$ that is spanned by the column space of $\X$ and offset by $\v$. 
Let $\calH_j:=\{H(\X,\v)\,|\X\in\calV_j,\v\in\mathbb{R}^d\}$ denote the set of all $j$-affine subspaces in $\mathbb{R}^d$.
\end{definition}
We use $\dist(H(\X,\v),\p):=\|(\p-\v)^\top\Y\|_2$ to denote the distance between any point $\p\in\mathbb{R}^d$ and the $j$-dimensional affine subspace $H(\X,\v)$, where here $\Y \in \mathbb{R}^{d \times (d - j)}$ such that $\Y^\top\X = 0^{(d-j) \times j}$ .

We now define the term \emph{query space} which will aid us in simplifying the proofs as well as the corresponding theorems.
\begin{definition}[query space]
Let $1\le j<d<n$ be positive integers and let $P\subseteq\mathbb{R}^d$ be a set of $n$ points such that $\Span(P)=\mathbb{R}^d$.
Then for the union of all $j$-affine subspaces $\calH_j$, the tuple $(P,\calH_j,\dist)$ is called a \emph{query space}.
\end{definition}

Following the previous definition, we now can define the notion of $L_\infty$ coreset and $L_2$ coreset.
\begin{definition}[$L_\infty$ coreset]
Let $j\in[d-1]$, $\eps\in(0,1)$, and $(P,\calH_j,\dist)$ be a query space.
Then a set $C \subseteq P$ is called an $L_\infty$ $\eps$-coreset with respect to the query space $(P,\calH_j,\dist)$ if for every $\X\in\calV_j$ and $\v\in\mathbb{R}^d$, $\max_{\p\in P}\dist(H(\X,\v),\p)\le(1+\eps)\max_{\p\in C}\dist(H(\X,\v),\p).$
\end{definition}

\begin{definition}[$L_2$ coreset]
Let $j\in[d-1]$, $\eps\in(0,1)$, and $(P,\calH_j,\dist)$ be a query space.
Then a set $C \subseteq P$ with a weight function $w:C\to\mathbb{R}$ is called an $L_2$ $\eps$-coreset with respect to the query space $(P,\calH_j,\dist)$ if for every $\X\in\calV_j$ and $\v\in\mathbb{R}^d$,
$\sum_{\p\in P}\dist(H(\X,\v),\p)^2\le(1+\eps)\sum_{\p\in C}w(\p)\dist(H(\X,\v),\p)^2$.
\end{definition}

Finally, we define a coreset for the $k$ $j$-cylinders problem, followed by the Carath\'{e}odory's theorem which will be used in our proofs and algorithms in computing the $L_\infty$ coreset for the $(k,j)$-projective clustering problem.

\begin{definition}
A closed $j$-cylinder of radius $r$ is a set of points in $\mathbb{R}^d$ whose distance to a certain $j$-flat is at most $r$.
A set $D$ is an $L_\infty$ $C$-coreset of $P\subseteq\mathbb{R}^d$ for the $(j,k)$-projective clustering problem if $D$ is a subset of $P$ such that there exists a union of $k$ $j$-cylinders of radius $Cr$ that covers $P$ for each union of $k$ $j$-cylinders of radius $r$ that covers $D$.
\end{definition}

\begin{theorem}[Carath\'{e}odory's theorem]
\citep{Caratheodory07,Steinitz13}
For any $A\subset\mathbb{R}^d$ and $\p\in\conv(A)$, there exists $m\le d+1$ points $\p_1,\ldots,\p_m\in A$ such that $\p\in\conv(\{\p_1,\ldots,\p_m\})$.
\end{theorem}

\section{$L_\infty$ CORESETS FOR PROJECTIVE CLUSTERING}
First, we note that~\cite{Har-Peled04} showed that $L_\infty$ coresets do not exist when the input set is $n$ points from $\mathbb{R}^d$. However in this paper, we consider the integer projective clustering problem, e.g.~\cite{EdwardsV05}, where the input points lie on a polynomial grid. 

We first give an $L_\infty$ coreset for the $(j,1)$-projective clustering problem in Section~\ref{sec:linfty:j1}. We then use our $L_\infty$ coreset for the $(j,1)$-projective clustering to inductively build an $L_\infty$ coreset for the $(j,k)$-projective clustering problem. 
\subsection{$L_\infty$ Coreset for $(j,1)$-Projective Clustering}
\label{sec:linfty:j1}
We first give an overview for our algorithm that produces a constant factor approximation coreset for the $(j,1)$-projective clustering problem. 
We again emphasize that our coreset for the $(j,1)$-projective clustering problem serves as our main technical contribution because we use Carath\'{e}odory's theorem to explicitly find a polynomial number of points to add to our coreset. 
We can then use a natural inductive argument to recursively add a polynomial number of points to create a coreset for the integer $(j,k)$-projective clustering problem. 
By contrast, even the base case for the only existing coreset for the integer $(j,k)$-projective clustering problem already contains an exponential number of points~\citep{EdwardsV05}. 

The algorithm takes as input a set $P\subseteq\mathbb{R}^d$ of $n$ points, which are promised to lie on a flat of dimension $j$, and computes a subset $C\subseteq P$, which satisfies Theorem~\ref{thm:single:projective}.
The algorithm appears in full detail in Algorithm~\ref{alg:single:projective} and first initializes $C$ to be an empty set.
Our algorithm computes $H(\W,\u)$ to be the $j$-dimensional flat that contains $P$ and sets $Q$ to be the set of points obtained by projecting $P$ onto the column space of $\W$. 
The algorithm then defines $E(\G,\c)$ to be the John-L\"{o}wner ellipsoid containing the convex hull of $Q$ and $S$ to be the set of vertices defined the axes of symmetry and the center of the scaled ellipsoid $\frac{1}{j}(E(\G,\c)-\c)+\c$, which can be explicitly and efficiently computed, and note that $|S|\le 2j$. 
From Carath\'{e}odory's theorem, we can express each point in $S\cup\{\c\}$ as a linear combination of $j+1$ points from $Q$.
We thus define $K$ to be the $\O{j^2}$ points of $Q$ needed to represent all points in $S\cup\{\c\}$ and set $C=\mu(K)$, where $\mu$ is the inverse mapping from $Q$ to $P$.
\begin{algorithm}[!htb]
\caption{Coreset for $(j,1)$-Projective Clustering}
\label{alg:single:projective}
\DontPrintSemicolon
\KwIn{$P\subseteq\mathbb{R}^d$ of $n$ points that lie on a flat of dimension $j$}
\KwOut{Coreset of size $\O{j^2}$}
$C\gets\emptyset$\;
Let $H(\W,\u) := $ a $j$-dimensional flat containing $P$\;
$Q:= \br{\W^\top\p\,|\,\p\in P}$\;
Let $\mu$ be function that maps each point $q\in Q$ to its original point in $P$\;
Let $E(\G,\c) := $ the John-L\"{o}wner ellipsoid of the convex hull of $Q$\;
$S := $ the vertices of the scaled ellipsoid $\frac{1}{j}\left(E(\G,\c)-\c \right)+\c$\;
\For{each $\s\in S\cup\{\c\}$}{
$K_{\s} := $ be at most $j+1$ points from $Q$ whose convex hull contains $\s$ \;
$C :=  C\cup\mu\left(K_{\s}\right)$\;
}
\Return{$C$}\;
\end{algorithm}
We first prove the following structural property that follows from Carath\'{e}odory's theorem.
\begin{lemma}
\label{lem:conv:ineq}
Let $d,\ell,m\ge 1$ be integers.
Let $\p\in\mathbb{R}^d$ and $A\subseteq\mathbb{R}^d$ be a set of $m$ points with $\p\in\conv(A)$ so that there exists $\alpha:A\to[0,1]$ such that $\sum_{\q\in A}\alpha(\q)=1$ and $\sum_{\q\in A}\alpha(\q)\cdot\q=\p$.
Then for every $\Y\in\mathbb{R}^{d\times\ell}$ and $\v\in\mathbb{R}^{\ell}$, $\|\p^\top\Y-\v\|_2\le\max_{\q\in A}\|\q^\top\Y-\v\|_2.$
\end{lemma}
\begin{proof}
Since we can write $\p$ as the convex combination of points $\q\in A$ with weight $\alpha(\q)$, we have
\[\|\p^\top\Y-\v\|_2=\|\left(\sum_{\q\in A}\alpha(\q)\q^\top\Y\right)-\v\|_2.\]
Moreover, we have $\sum_{\q\in A}\alpha(\q)=1$, so we can decompose $\v$ into
\[\|\p^\top\Y-\v\|_2=\|\sum_{\q\in A}\alpha(\q)\left(\q^\top\Y-\v\right)\|_2.\]
By triangle inequality,
\[\|\p^\top\Y-\v\|_2\le\sum_{\q\in A}\alpha(\q)\|\q^\top\Y-\v\|_2\le\max_{\q\in A}\|\q^\top\Y-\v\|_2.\]
\end{proof}

We use Lemma~\ref{lem:conv:ineq} to show that Algorithm~\ref{alg:single:projective} gives a coreset for the $(j,1)$-projective clustering problem as summarized below. In addition, we show that our $\ell_\infty$-coreset is also applicable towards the $\left(j,z\right)$-clustering where $j$ denotes the dimensionality of the subspace, and $z$ denotes the power of the distance function. For instance, $z \in [1,2)$ is used for obtaining robust clustering, which is useful against outliers.
\begin{theorem}
\label{thm:single:projective}
Let $j\in[d-1]$, $z\ge 1$, and let $(P,\calH_j,\dist)$ be a query space, where $P$ lies in a $j$-dimensional flat.
Let $C\subseteq P$ be the output of Algorithm~\ref{alg:single:projective}.
Then $|C|=\O{j^2}$ and for every $H(\X,\v)\in\calH_j$, we have
$\max_{\p\in P}\dist(\p,H(\X,\v))^z\le 2^{z+1}j^{1.5z}\max_{\q\in C}\dist(\q,H(\X,\v))^z.$
\end{theorem}
\begin{proof}
To show the first part of the claim, note that since the ellipsoid $E(\G,\c)$ has at most $2j$ vertices and each vertex point of the ellipsoid can be represented a convex combination fo at most $j+1$ points from $Q$ by Carath\'{e}odory's theorem, then the number of points in $C$ is at most $2(j+1)^2$, so that $|C|=\O{j^2}$.

To show the second part of the claim, we first set $H(\W,\u)$ to be the $j$-flat containing $P$ and $\Y\in\calH_{d-j}$ so that $\Y^\top\X=0^{(d-j)\times j}$ and $\X^\top\Y=0^{j\times(d-j)}$.
Notice that each $\p\in P$ satisfies
\[\dist(p,H(\X,\v))^z=\|(\p-\v)^\top\Y\|_2^z=\|(\p-\u+\u-\v)^\top\Y\|_2^z.\]
Since $\p$ lies in the affine flat $H(\W,\u)$, then we have
\begin{align}
\label{eqn:bound}
\dist(p,H(\X,\v))^z=\|\left(\W\W^\top(\p-\u)+\u-\v\right)^\top\Y\|_2^z.
\end{align}
We now rely on properties of Carath\'{e}odory's Theorem and the John-L\"{o}wner ellipsoid to bound (\ref{eqn:bound}). 
First note that
\[\|\left(\W\W^\top(\p-\u)+\u-\v\right)^\top\Y\|_2^z = \|\left(\W\W^\top\p-\W\W^\top\u+\u-\v\right)^\top\Y\|_2^z.\]

Recall that for each $\p\in P$, there exists $\q\in Q$ such that $\q=\W^\top\p$ and
\[\|\left(\W\W^\top(\p-\u)+\u-\v\right)^\top\Y\|_2^z=\|\left(\W\q-\W\W^\top\u+\u-\v\right)^\top\Y\|_2^z.\]
Since $S$ is the set of vertices of $E(\G,\c)$, we have by the definition of the John-L\"{o}wner ellipsoid that
\[\frac{1}{j}(E(\G,\c)-\c)+\c\subseteq\conv(Q)\subseteq E(\G,\c).\]
Thus $S\subseteq\conv(S)\subseteq\conv(Q)$ and by Carath\'{e}odory's theorem, for each $\s\in S$, there exists a set $K_\s$ of at most $j+1$ points such that $\s\in\conv(K_\s)$.
By Lemma~\ref{lem:conv:ineq},
\[\|\s\W^\top\Y\|_2^z\le\max_{\q\in K_\s}\|q^\top\W^\top\Y\|_2^z.\]
We also have
\[\frac{1}{\sqrt{j}}\cdot\frac{E(\G,\c)-\c}{j}+\c\subseteq\conv(S)\subseteq\frac{E(\G,\c)-\c}{j}+\c.\]
Therefore,
\begin{align}
\label{eqn:conv:contain}
\conv(S)\subseteq\conv(Q)\subseteq E(\G,\c)\subseteq j^{1.5}(\conv(S)-\c)+\c.
\end{align}
Thus for every $\q\in Q$, there exists $\s\in\conv(S)$ and $\gamma\in[0,1]$ such that
\[\q=\gamma\s+(1-\gamma)(j^{1.5}(\s-\c)+\c).\]
For $\a=\u^\top\W\W^\top\Y-\u^\top\Y-\v^\top\Y$, we then have
\[\|\q^\top\W^\top\Y+\a\|_2^z=\|(\gamma\s+(1-\gamma)(j^{1.5}(\s-\c)+\c)\W^\top\Y+\a\|_2^z.\]
Since $z\ge 1$, then $\|\cdot\|_2^z$ is a convex function.
Thus by Jensen's inequality,
\[\|\q^\top\W^\top\Y+\a\|_2^z \le \gamma\|\s\W^\top\Y+\a\|_2^z+(1-\gamma)\|j^{1.5}\s^\top\W^\top\Y+(1-j^{1.5})\c^\top\W^\top\Y+\a\|_2^z.\]

Since $\a=j^{1.5}\a+(1-j^{1.5})\a$, then
\[\|j^{1.5}\s^\top\W^\top\Y+(1-j^{1.5})\c^\top\W^\top\Y+\a\|_2^z\le2^zj^{1.5z}\|\s^\top\W^\top\Y+\a\|_2^z+2^z(j^{1.5}-1)^z\|\c^\top\W^\top\Y+\a\|_2^z.\]

Since $\c\in\conv(S)$ by \eqref{eqn:conv:contain}, then
\[\|\c^\top\W^\top\Y+\a\|_2^z\le\max_{\s\in\conv(S)}\|\s^\top\W^\top\Y+\a\|_2^z.\]
Since $j^{1.5z}+(j^{1.5}-1)^z\le2j^{1.5z}$, then we have that for every $\q\in Q$,
\[\|\q^\top\W^\top\Y+\a\|_2^z\le2^{z+1}j^{1.5z}\max_{\s\in S}\|\s^\top\W^\top\Y+\a\|_2^z.\]
Thus we have for every $\s\in S$,
\begin{equation*}
\begin{split}
\|\s^\top\W^\top\Y+\a\|_2^z&\le\max_{\q\in K}\|\q^\top\W^\top\Y+\a\|_2^z\\
&\le\max_{\p\in C}\|\p^\top\W\W^\top\Y+\a\|_2^z
\end{split}
\end{equation*}
Because $\a=\u^\top\W\W^\top\Y-\u^\top\Y-\v^\top\Y$, we have
\begin{equation*}
\begin{split}
\dist(\p,H(\X,\v))^z&\le 2^{z+1}j^{1.5z}\max_{\p\in C}\|\p^\top\W\W^\top\Y+\a\|_2^z\\
&=2^{z+1}j^{1.5z}\max_{\p\in C}\|(\W\W^\top(\p-\u))^\top\Y+\u^\top\Y-\v^\top\Y\|_2^z.
\end{split}
\end{equation*}

Since $(\p-\u)\in P$ and $P$ lies within $H(\W,\u)$, then
\begin{align*}
\dist(\p,H(\X,\v))^z&\le2^{z+1}j^{1.5z}\max_{\p\in C}\|\p^\top\Y-\u^\top\Y+\u^\top\Y-\v^\top\Y\|_2^z\\
&=\max_{\p\in C}\|\p^\top\Y-\v^\top\Y\|_2^z\\
&=2^{z+1}j^{1.5z}\max_{\p\in C}\dist(\p,H(\X,\v))^z.
\end{align*}
\end{proof}

\subsection{$L_\infty$ Coreset for $(j,k)$-Projective Clustering}
Our coreset construction is recursive. Generally speaking, we construct a coreset $D_k$ for $(j,k)$-projective clustering from a coreset $D_{k-1}$ for $(j,k-1)$-projective clustering. For the base case, we show how to construct a coreset $D_1$ for $(j,1)$-projective clustering in Theorem~\ref{thm:single:projective}. 
Now for $k\ge 2$, given a coreset $D_{k-1}\subset P$ for $(j,k-1)$-projective clustering, the construction of $D_k$ has $j+1$ levels and the $i$-th level will specify $i+1$ points $\v_0,\ldots,\v_i$ and a corresponding point set $P[\v_0,\ldots,\v_i]\subset P$. 
We first add $D_{k-1}$ into $D_k$ and separately initialize Level 0 with $\v_0$ being each point of $D_{k-1}$ and define $P[\v_0]=P$. Crucially, each of the $j+1$ levels only adds to the coreset a number of points that is polynomial in $j\le d-1$ at each level due to the base case using our new coreset for $(j,1)$-projective clustering based on Carath\'{e}odory's theorem. 
Hence, the total number of points is polynomial in $d$ but exponential in $j$. 
By contrast, existing coreset constructions of \cite{EdwardsV05} use partitions that must be analyzed over $d$ levels due to their lack of an efficient coreset for their base case; thus their size is exponential in $d$. 

\textbf{Level 0:}
Given any choice of $\v_0$ from $D_{k-1}$, we define $P[\v_0]:=P\subset[\Delta]^d$, we have $\dist(\p,\v_0)\in[1,\Delta\sqrt{d}]$ for every $\p\in P[\v_0]$.
We can partition $P[\v_0]$ into $\ell=\O{\log(d\Delta)}$ sets $K_{0,0},K_{0,1},\ldots,K_{0,\ell}$ such that $K_{0,0}=\{\v_0\}$ and $K_{0,i}=\{\p\in P[\v_0]\,:\,2^{i-1}\le\dist(\p,\v_0)\le2^i\}$ for $i\ge 1$.
Intuitively, this can be seen as partitioning the points of $P$ into sets with exponentially increasing distance from $\v_0$.
For each $K_{0,i}$, we construct an $L_\infty$-coreset $D_{0,i}$ of $K_{0,i}$ for the $(j,k-1)$-projective clustering problem and add $D_{0,i}$ into $D_k$.
We then separately select $\v_1$ to be any point in $D_{0,i}$ across all $i\in[\ell]$ and set $P[v_0,v_1]=\cup_{x=0}^i K_{0,x}$.

\textbf{Level $t$, for $t\in[1,j]$:}
Given $\v_0,\ldots,\v_t$ and $P[\v_0,\ldots,\v_t]$, let $A_t$ denote the affine subspace spanned by $\v_0,\ldots,\v_t$.
We recall the following structural properties about the convex hull of affine subspaces.
\begin{lemma}[\citep{EdwardsV05}, Lemma 45 in~\citep{FeldmanSS20}]
\label{lem:disc:affine}
Let $\Delta\ge 2$, $k$ be a positive integer, and $j\le d-1$ be a positive integer.
Let $\calQ_{j,k}$ be the family of all sets of $k$ affine subspaces of $\mathbb{R}^d$ with dimension $j$.
Let $\A\in\{-\Delta,\ldots,\Delta\}^{n\times d}$.
Then for every $H\in\calH_j$, we have either $\dist(H,\A)=0$ or $\dist(H,\A)\ge\frac{1}{(d\Delta)^{cj}}$, for some universal constant $c>0$.
\end{lemma}
By Lemma~\ref{lem:disc:affine}, we have that for every $\p\in P[\v_0,\ldots,\v_t]$, that $\dist(\p,A_t)$ is either $0$ or in the range $[1/(d\Delta)^{cj},2\Delta\sqrt{d}]$.
Thus we can once again partition $P[\v_0,\ldots,\v_t]$ into $\O{j\log(d\Delta)}$ subsets $K_{t,0},\ldots,K_{t,\ell}$ such that $K_{t,0}=P[\v_0,\ldots,\v_t]\cap A_t$ and for each integer $i\in[\ell]$, $K_{t,i}:=\{\p\in P[\v_0,\ldots,\v_t]\,:\,2^{i-1}c_j/\Delta^j\le\dist(\p,A_t)<2^i c_j/\Delta^j\}.$
For each $K_{t,i}$, we construct an $L_\infty$-coreset $D_{t,i}$ of $K_{t,i}$ for $(j,k-1)$-projective clustering and add $D_{t,i}$ to $D_k$.
We then separately select $\v_{t+1}$ to be any point in $D_{t,i}$ across all $i\in[\ell]$ and set $P[v_0,\ldots,\v_{t+1}]=\cup_{x=0}^i K_{t,x}$. We remark that we terminate at level $j+1$. Finally, in what follows, we give a bound on the size of our $L_\infty$ coreset.

\begin{restatable}[Coreset size]{lemma}{lemcoresetsize}
\label{lem:coresetsize}
Let $f(k)=|D_k|$ denote the size of the coreset $D_k$ formed at level $k$ for $(j,k)$-projective clustering.
Then $f(k)=\left(8j^3\log(d\Delta)\right)^{\O{jk}}$.
\end{restatable}
\begin{proof}
By Theorem~\ref{thm:single:projective}, we have that $f(1)\le2(j+1)^2\le8j^2$.
Our construction has $j+1$ levels and each level partitions the data set into $\O{j\log(d\Delta)}$ sets.
For each of the sets, we construct an $L_\infty$-coreset for $(k-1,j)$-projective clustering and each of the points in the union of the coresets to be used in the point set $P[\v_0,\ldots,\v_{k+1}]$ for the next level.
Thus we have
\[f(k)\le(\O{j\log(d\Delta)}\cdot f(k-1))^{j+1},\]
so that by induction, $f(k)\le(8j^3\log(d\Delta))^{\O{jk}}$.
\end{proof}

\noindent
To prove that our construction yields an $L_\infty$ constant-factor approximation coreset for the integer $(j,k)$-projective clustering problem, we use a structural property about the convex hull of affine subspaces. 
Informally, the property says that if $\v_0,\ldots,\v_d\in\mathbb{R}^d$ are $d+1$ affinely independent vectors that induce a sequence of affine subspaces $\A_0,\ldots,\A_d$, then under certain assumptions, the convex hull formed by $\v_0,\ldots,\v_d$ contains a translation of a scaled hyperrectangle formed by a sequence $\u_0,\ldots,\u_d$ of vectors formed by the orthogonal projection away from $\A_0,\ldots,\A_d$. 

\begin{lemma}[Lemma 1 in~\citep{EdwardsV05}]
\label{lem:affsub:convrect}
Let $\v_0,\ldots,\v_d\in\mathbb{R}^d$ be $d+1$ affinely independent vectors and for each $0\le i\le d$, let $A_i$ be the affine subspace spanned by $\v_0,\ldots,\v_i$.
Let $\w_i$ be the projection of $\v_i$ onto $A_i$ and let $\u_i=\v_i-\w_i$.
Suppose we have $\dist(\v_j,A_i)\le2\|\u_i\|_2$ for every $0\le i\le d$ and $j\ge i$.
Then there exists an absolute constant $c_d$ that only depends on $d$, so that the simplex $\conv(\v-0,\ldots,\v_d)$ contains a translation of the hyperrectangle
$\{c_d(\alpha_1\u_1+\ldots+\alpha_d\u_d\,:\,\alpha_i\in[0,1]\}.$
\end{lemma}
Using this structural property, we achieve an $L_\infty$ constant-factor approximation coreset for the integer $(j,k)$-projective clustering problem with size $(8j^3\log(d\Delta))^{\O{jk}}$:
\begin{restatable}{lemma}{lemcoresetapprox}
\label{lem:coreset:approx}
There exists a universal constant $\xi>0$ such that $D_k$ is a $\xi$-coreset for the $(j,k)$-projective clustering problem.
\end{restatable}
\begin{proof}
Suppose $D_k$ is covered by the $k$ cylinders $S_1,\ldots,S_k$.
Then we would like to show that $P$ is covered by a constant-factor $C$-expansion of $S_1,\ldots,S_k$. Here a $x$-expansion of a cylinder $S$ is the set $\left\lbrace x p \middle| p \in S \right\rbrace$.
We first induct on $k$ and then $j$, noting that the base case $k=1$ is already handled by Theorem~\ref{thm:single:projective}.
We then fix $k\ge 2$ and induct on $j$, first considering stage $0$, where we have some $\v_0$ and we define $K_{0,i}=\{\p\in P[\v_0]\,:\,2^{i-1}\le\dist(\p,\v_0)\le 2^i\}$ for $i\in[\ell]$, where $\ell=\O{\log(d\Delta)}$.
We then set $D_{0,i}$ to be the corresponding coreset for $K_{0,i}$ for the $(k-1,j)$-projective clustering problem.
Let $a$ denote the largest positive integer such that $S_k\cap D_{0,a}\neq\emptyset$, so that by the definition of $a$, we have that $\cup_{x=a+1}^{\ell}D_{0,x}$ is covered by $S_1,\ldots,S_{k-1}$.
Since $D_{0,x}$ is a coreset for the $(k-1,j)$-projective clustering problem, then $\cup_{x=a+1}^\ell K_{0,x}$ is covered by a $C$-expansion of $S_1,\ldots,S_{k-1}$.
For any point $\v_1$ in $S_k\cap D_{0,a}$, we enter stage $1$ with $\v_0,\v_1$ and so it remains to prove that a $C$-expansion of $S_1,\ldots,S_k$ covers $P[\v_0,\v_1]=\cup_{x=0}^aK_{0,x}$.

For the inductive step, suppose we have fixed $\v_0,\ldots,\v_t$ and for each $i\in[0,t]$, let $A_i$ denote the affine subspace spanned by $\v_0,\ldots,\v_i$, that is $A_i = \left\lbrace\sum\limits_{l = 0}^i \alpha_l v_l \middle| \forall l \in [i] \, \alpha_i \in \mathbb{R}, \sum\limits_{l=0}^i \, \alpha_l = 1 \right\rbrace$.
Let $\w_i$ denote the projection of $\v_i$ on $A_i$ and set $\u_i=\v_i-\w_i$.
Then for every $\p\in P[\v_0,\ldots,\v_i]\cap A_i$, we have
\[\dist(\p,A_i)\le2\dist(\v_i,A_i).\]
Thus for $\p\in P[\v_0,\ldots,\v_t]\cap A_t$, we have that $\p$ is contained in the hyperrectangle
\[\calM:=\v_0+\{\alpha_1\u_1+\ldots+\alpha_t\u_t\,:\,\alpha_i\in[-2,2]\}.\]
By Lemma~\ref{lem:affsub:convrect}, there exists a constant $c_t$ such that $\conv(\v_0,\ldots,\v_t)$ contains a translation of the hyperrectangle
\[\calM_1:=\{c_t(\alpha_1\u_1+\ldots+\alpha_t\u_t)\,:\,\alpha_i\in[0,1]\}.\]
Since $S_k$ covers $\v_0,\ldots,\v_t$, then $\calM_1\subset S_k$.
Moreover, we have that for an absolute constant $\xi$, $\calM\subset\xi\cdot\calM_1$.
Thus, a $\xi$-expansion of $S_k$ covers $P[\v_0,\ldots,\v_t]\cap A_t$.

Let $b$ denote the largest positive integer such that $S_k\cap D_{t,b}\neq\emptyset$, so that by the definition of $b$, we have that $\cup_{x=b+1}^{\ell}D_{t,x}$ is covered by $S_1,\ldots,S_{k-1}$.
Since $D_{t,x}$ is a coreset for the $(k-1,j)$-projective clustering problem, then $\cup_{x=b+1}^\ell K_{t,x}$ is covered by a $\xi$-expansion of $S_1,\ldots,S_{k-1}$.
For any point $\v_{t+1}$ in $S_k\cap D_{t,b}$, we enter stage $t+1$ with $\v_0,\ldots,\v_{t+1}$ and so then by induction, it holds that a $\xi$-expansion of $S_1,\ldots,S_k$ covers $P[\v_0,\ldots,\v_{t+1}]=\cup_{x=0}^b K_{t,x}$.
\end{proof}

Theorem~\ref{thm:main:infty} then follows from Lemma~\ref{lem:coreset:approx} and Lemma~\ref{lem:coresetsize} and the observation that $j\le d-1$. 
Thus our coresets have size polynomial in $d$, resolving the natural open question from \cite{EdwardsV05}. 


\begin{algorithm}[!htb]
\caption{Coreset for $(j,k)$-Projective Clustering}
\label{alg:full:projective}
\DontPrintSemicolon
\KwIn{$P\subseteq\mathbb{R}^d$ of $n$ points, an integer $j\in[d-1]$, an integer $k\ge 1$, an accuracy parameter $\eps\in(0,1)$ and a failure probability $\delta\in(0,1)$.}
\KwOut{A weighted set $(C,u)$}
$P_1 := P$, $i := 1$, $C := \emptyset$ \;
\While{$\left|P_i\right| \ge 1$}{
$S_i := $ an $L_\infty$-coreset for $(j,k)$-projective clustering\;
\For{every $\p\in S_i$}{
$s(\p) := \frac{1}{i}\cdot\left|S_i\right|$ \tcp{$|S_i|=\O{j^{1.5}(j\log(d\Delta))^{\O{jk}}}$}
}
$P_{i+1} :=  P_i\setminus S_i$, $i := i+1$\;
}
$t := \sum_{\p\in P}s(\p)$ \tcp{$t=\O{j^{1.5}(j\log(d\Delta))^{\O{jk}}\log n}$}
$m := \frac{ct}{\eps^2}\left(djk\log\frac{t}{\delta}\right)$\;
\For{$m$ iterations}{
Sample a point $\p \in P$ with probability $\frac{s(\p)}{t}$\;
$C := C\cup\{\p\}$, $u(\p):= \frac{t}{m\cdot s(\p)}$\;
}
\Return{$(C,u)$}\;
\end{algorithm}

\section{APPLICATIONS}
In this section, we show that our framework gives an $L_\infty$ coreset for subspace clustering, as well as a large class of $M$-estimators. 
To the best of our knowledge, our constructions are the first coresets with size polynomial in $d$ for these $M$-estimators. 
Namely, our algorithm achieves approximate regression for the Cauchy, Welsch, Huber, Geman-McClure, Tukey, $L_1-L_2$, Fair loss functions, as well as general loss functions that are concave or power-bounded; see Table~\ref{table:Mestimators}.

\textbf{Beyond traditional projective clustering.} First, we present that our $L_\infty$-coreset algorithm is applicable for a family of non-decreasing log-log Lipschitz function.

\begin{restatable}[$L_\infty$ coreset for log-log Lipschitz loss functions]{theorem}{thmmainapps}
\label{thm:main:apps}
Let $j\in[d-1]$, $z\ge 1$, and let $(P,\calH_j,\dist)$ be a query space, where $P$ lies in a $j$-dimensional flat.
Let $f:[0,\infty)\to[0,\infty)$ such that both (1) $f$ is a monotonically non-decreasing function, i.e., for every $x,y\in[0,\infty)$ with $x\le y$, it holds that $f(x)\le f(y)$ and (2) $f$ is log-log Lipschitz, i.e., there exists $\rho\ge 1$ for every $b\ge 1$ such that $f(bx)\le b^\rho f(x)$.
Let $C$ be the output of a call to $L_\infty-\coreset(P)$.
Then for every $H\in\calH_j$, $\max_{\p\in P}f(\dist(p,H(\X,\v))^z)\le (2^{z+1}j^{1.5z})^\rho\max_{\p\in C}f(\dist(p,H(\X,\v))^z).$
\end{restatable}
\begin{proof}
Let $H(\X,\v)\in\calH_j$.
Then by Theorem~\ref{thm:single:projective}, we have that
\[\max_{\p\in P}\dist(p,H(\X,\v))^z\le 2^{z+1}j^{1.5z}\max_{\q\in C}\dist(q,H(\X,\v))^z.\]
Since $f$ is a monotonically non-decreasing function, then
\begin{align*}
\max_{\p\in P}\,f(\dist(p,H(\X,\v))^z)&=f\left(\max_{\p\in P}\dist(p,H(\X,\v))^z\right)\\
&\le f\left(2^{z+1}j^{1.5z}\max_{\q\in C}\dist(q,H(\X,\v))^z\right).
\end{align*}
Since $f$ is log-log Lipschitz, then
\begin{align*}
f\left(2^{z+1}j^{1.5z}\max_{\q\in C}\dist(q,H(\X,\v))^z\right)&\le(2^{z+1}j^{1.5z})^\rho f\left(\max_{\q\in C}\dist(q,H(\X,\v))^z\right)\\
&\le (2^{z+1}j^{1.5z})^\rho\max_{\q\in C} f\left(\dist(q,H(\X,\v))^z\right).
\end{align*}
Hence, we have
\begin{align*}
\max_{\p\in P}&f(\dist(p,H(\X,\v))^z) \le(2^{z+1}j^{1.5z})^\rho\max_{\p\in C}f(\dist(p,H(\X,\v))^z)
\end{align*}
as desired.
\end{proof}

Although the above theorem is applicable to large family of functions, it may not yield tight bounds for each of the loss functions in Table~\ref{table:Mestimators}. 
Thus we first prove the following lemma, which guarantees an coreset for power-bounded loss functions $\Psi_{Pow}(x)$.

\begin{restatable}[$L_\infty$ coreset for regression with power-bounded loss function]{lemma}{lempower}
\label{lem:power}
Let $P\subseteq\mathbb{R}^d$ be a set of $n$ points, $b:P\to\mathbb{R}$, $\lambda\in\mathbb{R}$, and let $z>0$ be a fixed constant. 
Let $\Psi_{Pow}$ denote any non-decreasing loss function with $\Psi_{Pow}(0)=0$ and $\Psi_{Pow}(y)/\Psi_{Pow}(x)\le(y/x)^z$ for all $0\le x\le y$.  
Let $P'=\{\p\circ b(\p)\,|\,\p\in P\}$, where $\circ$ denotes vertical concatenation.
Let $C'$ be the output of a call to $L_\infty-\coreset(P',d)$ and let $C\subseteq P$ so that $C'=\{\q\circ b(\q)\,|\,\q\in C\}$.
Then for every $\w\in\mathbb{R}^d$, $\max_{\p\in P}\Psi_{Pow} \left(|\p^\top\w-b(\p)|\right)\le 4^z(d+1)^{1.5z}\cdot\max_{\q\in C}\Psi_{Pow}\left(|\q^\top\w-b(\q)|\right).$
\end{restatable}
\begin{proof}
Because the claim is trivially true for $\w=0^d$, then it suffices to consider nonzero $\w\in\mathbb{R}^d$. 
Let $Y\in\calH_{d-1}$ such that $\w^\top\Y=0^{d-1}$ and $\Y^\top\w=0^d$.
For each $\p\in P$, let $\p'=\p\circ b(\p)=\begin{bmatrix}\p\\ b(\p)\end{bmatrix}$ denote the vertical concatenation of $\p$ with $b(\p)$.
We also define the vertical concatenation $\w'=\w\circ(-1)=\begin{bmatrix}\w\\ -1\end{bmatrix}$.
By setting $C$ to be the output of $\coreset$ on $P'=\{\p'\,|\,\p\in P\}$, then by Theorem~\ref{thm:main:apps},
\begin{align*}
\max_{\p\in P}&\dist(\p',H(\w',0^{d+1}))^z\\
&\le2^{z+1}(d+1)^{1.5z}\max_{\q\in C}\dist(\q',H(\w',0^{d+1}))^z
\end{align*}
Thus for $z=1$, we have for every $\p\in P$,
\[|(\p')^\top\w'|\le4(d+1)^{1.5}\max_{\q\in C}|(\q')^\top\w'|.\]
Since $\Psi_{Pow}$ is monotonically non-decreasing, then $\Psi_{Pow}(|\p^\top\x-b(\p))$ increases as $|\p^\top\x-b(\p)|$ increases. 
Moreover, we have $\Psi_{Pow}(y)/\Psi_{Pow}(x)\le(y/x)^z$ for all $0\le x\le y$. 
Therefore,
\begin{align*}
\max_{\p\in P}\Psi_{Pow}\left(|\p^\top\w-b(\p)|\right)
\le\max_{\q\in C}\Psi_{Pow}\left(4(d+1)^{1.5}|\q^\top\w-b(\q)|\right)
\le4^z(d+1)^{1.5z}\max_{\q\in C}\Psi_{Pow}\left(|\q^\top\w-b(\q)|\right).
\end{align*}
\end{proof}

Since power-bounded loss functions satisfy the conditions of Theorem~\ref{thm:main:apps}, then we can immediately apply Theorem~\ref{thm:main:apps} to obtain a base case for $z=1$. 
Lemma~\ref{lem:power} then follows by the definition of power-bounded loss functions for general $z$. 

It turns out that many of the loss functions of interest in Table~\ref{table:Mestimators} are power-bounded loss functions with specific parameters, so we can apply Theorem~\ref{thm:main:apps} in the same way as the proof of Lemma~\ref{lem:power} to obtain the guarantees for Cauchy regression, Huber regression, and Gem-McClure regression. 
However, in certain cases, we can prove structural properties bounding the growth of these loss functions to obtain guarantees that are sharper than those provided by Theorem~\ref{thm:main:apps}. 
We prove such structural properties at Section~\ref{supplement:app} of the supplementary material to handle Welsch regression, regression with concave loss functions, Tukey regression, $L_1-L_2$ regression, and Fair regression.  

\textbf{$L_\infty$-coreset to $L_2$-coreset for integer $(j,k)$-projective clustering.}
To construct an $\eps$-coreset, we use our $L_\infty$ coreset along with the framework of sensitivity sampling, in which points are sampled according to their sensitivity, a quantity that roughly captures how important or unique each point is. We give the coreset construction in Algorithm~\ref{alg:full:projective} using a standard reduction from an $L_2$ coreset to an $L_{\infty}$ coreset based on sensitivity sampling as summarized below.
\begin{restatable}{theorem}{thmltwomain}
\label{thm:ltwo:main}
With constant probability, Algorithm~\ref{alg:full:projective} outputs an $L_2$ $(1+\epsilon)$-coreset for $(j,k)$-projective clustering of $P$. 
\end{restatable}
\begin{proof}
The coreset size follows the bound of \cite{FeldmanSS20} once the sensitivity and the shattering dimension upper bound are given to us. 
We actually follow the way of Lemma 3.1 of \cite{Varadarajan2012ANA} to give the sensitivity upper bound $s(p)$. 
The shattering dimension upper bound $\tilde{O}(djk)$ follows Corollary 34 of \cite{FeldmanSS20}
\end{proof}

\textbf{Time complexity of our methods.} The running time of Algorithm~\ref{alg:full:projective}, we need to handle two cases -- \begin{enumerate*}
\item $k = 1$, and 
\item $k > 1$.
\end{enumerate*}
Observe that the time needed for constructing our $L_2$-coreset for $(k,j)$-projective clustering where $k = 1$ and any $j \geq 2$ is bounded by $O\left( n \left( n + j^4\log{n} \right)\right)$ time. Specifically speaking, the time depends heavily on the time that Algorithm~\ref{alg:single:projective}. Algorithm~\ref{alg:single:projective} depends heavily on the computation of the L\"{o}wner ellipsoid and on applying Carath\'{e}odory's theorem. The time needed to compute the L\"{o}wner ellipsoid of a given set of point $Q \subseteq \mathbb{R}^j$ such that $|Q| = n$ is bounded by $O\left( nj^3\log{n}\right)$~\citep{todd2007khachiyan}. 
As for constructing the Caratheodory set, recently~\cite{maalouf2019fast} provided an algorithm for computing such set in time $O\left(nj + j^4\log{n}\right)$. Combining these two methods with the observation that Algorithm~\ref{alg:full:projective} has $O\left( \frac{n}{j^2}\right)$ calls to Algorithm~\ref{alg:single:projective}, results in the upper bound above.

As for $k \geq 2$, following our analyzed steps needed to construct an $L_\infty$-coreset for the $(k,j)$-projective clustering problem and its variants, the running time is bounded from above by $O\left( nj^4\left(\log{\Delta}\right)^{j^2k}\right)$. Hence, Algorithm~\ref{alg:full:projective} requires $O\left( n^2j^4\left(\log{\Delta}\right)^{j^2k}\right)$ to construct an $L_2$-coreset for the $(k,j)$-projective clustering problem.

We note that our algorithm can be boosted theoretically speaking via the use of the merge-and-reduce tree~\cite{feldman2020core}, resulting in an algorithm  that are near-linear in $n$ rather than quadratic in $n$.

We further note that, our assumption on $P$ being contained in some $j$-dimensional affine subspace can be dropped as follows.
\begin{remark} 
So far, $P$ was assumed to lie on $j$-dimensional subspaces, however, one can remove this assumption by using Theorem $7$ of~\cite{varadarajan2012sensitivity}.
\end{remark}

\textbf{Subspace clustering.}
We first recall that subspace clustering is a variant of projective clustering where $k=1$ and $j\in[d-1]$.

\textbf{$M$-estimator regression.}
We present various robust $(1,d-1)$-projective clustering problems for which a strong $\eps$-coreset can be generated using our algorithms. 
We are given a set $P$ of $n$ points in $\mathbb{R}^d$ and a function $b:P\to\mathbb{R}$, and our goal is to optimize the minimization problem $\min_{\w\in\mathbb{R}^d}\sum_{\p\in P}\Psi\left(|\p^\top\w-b(\p)|\right)$, where $\Psi$ is any loss function. 
In particular, the choice of $\Psi$ encompasses many robust regression loss functions that have been designed to reduce the effect of outliers across various optimization problems. 
We show that Algorithm~\ref{alg:single:projective} achieves an $L_\infty$-coreset with accuracy $1-\frac{1}{\poly(d)}$ for a variety of loss functions; See Section~B in the supplementary material.

\section{EXPERIMENTS}
\label{sec:results}
\begin{table*}[htb!]
\caption{\textbf{Summary of our results: } Our coreset construction was applied on various application of projective clustering, of which were robust regression as well as robust subspace clustering}
\centering
\begin{tabular}{|c|c|c|c|c|c|}
\hline
Problem type & Loss function & $k$ & $j$ &  Dataset & Figure \\
\hline
Regression & Huber & $1$ & $d-1$ & \ref{dataset:1} & \ref{fig:reg_synth_huber} \\
\hline
Regression & Cauchy & $1$ & $d-1$ & \ref{dataset:1} & \ref{fig:reg_synth_cauchy} \\ \hline
$(2,2)$-projective clustering & $L_2^2$ & $2$ & $2$ & \ref{dataset:2} & \ref{fig:proj_hour_2_2_l2}\\
\hline
Robust $(2,2)$-projective clustering & Cauchy & $2$ & $2$ & \ref{dataset:2} & \ref{fig:proj_hour_2_2_Cauchy}\\ \hline
Robust $(2,2)$-projective clustering & Tukey & $2$ & $2$ & \ref{dataset:3} & \ref{fig:proj_CASP_2_2_Tukey}\\
\hline
Robust $(2,2)$-projective clustering & Welsch & $2$ & $2$ & \ref{dataset:3} & \ref{fig:proj_CASP_2_2_Welsch}\\
\hline
\end{tabular}
\label{tab:summary_results}
\end{table*}

\begin{figure*}[tbh!]
\centering
\begin{subfigure}[t]{0.49\textwidth}
\centering
\includegraphics[width=0.49\textwidth]{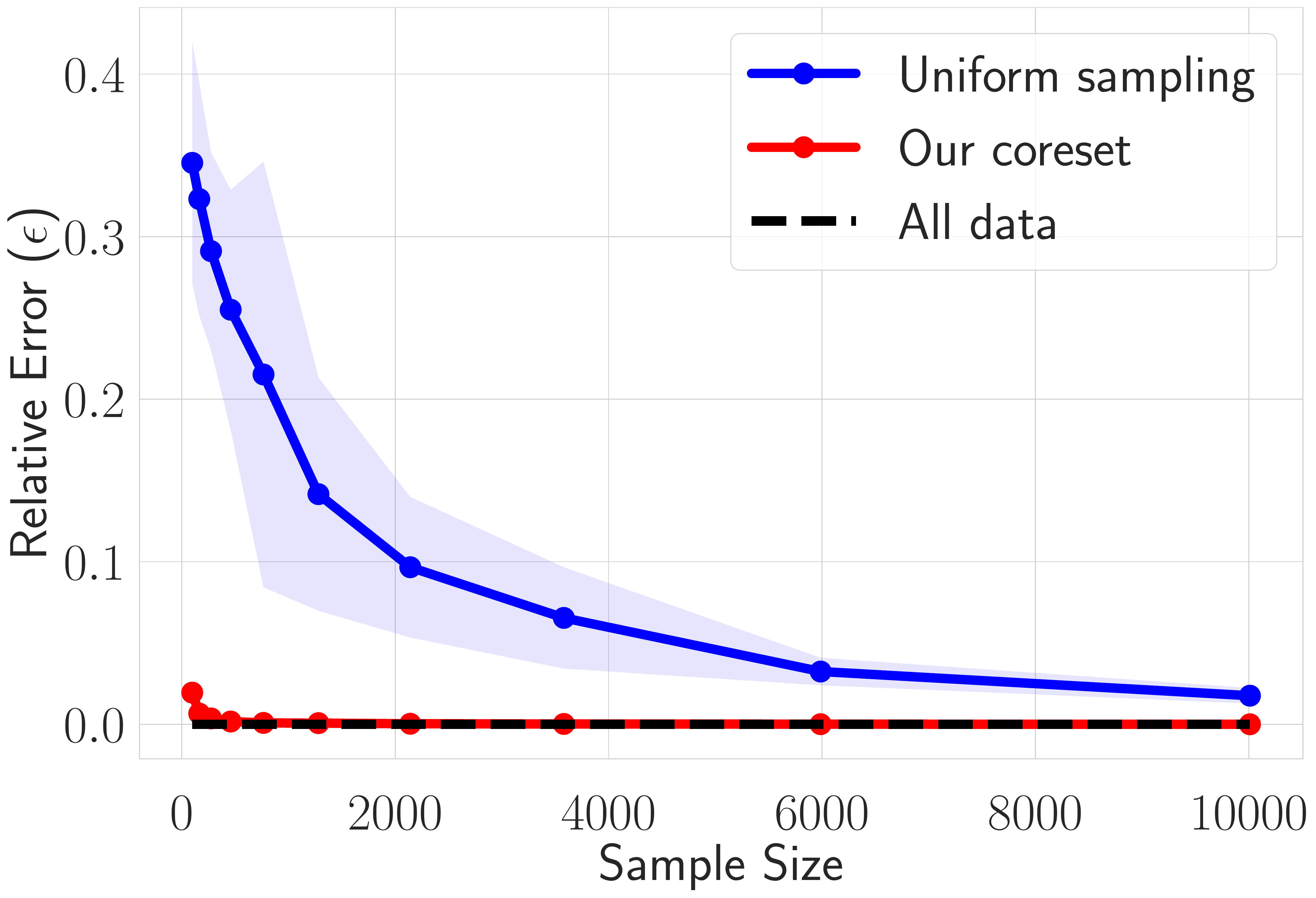}
\includegraphics[width=.49\textwidth]{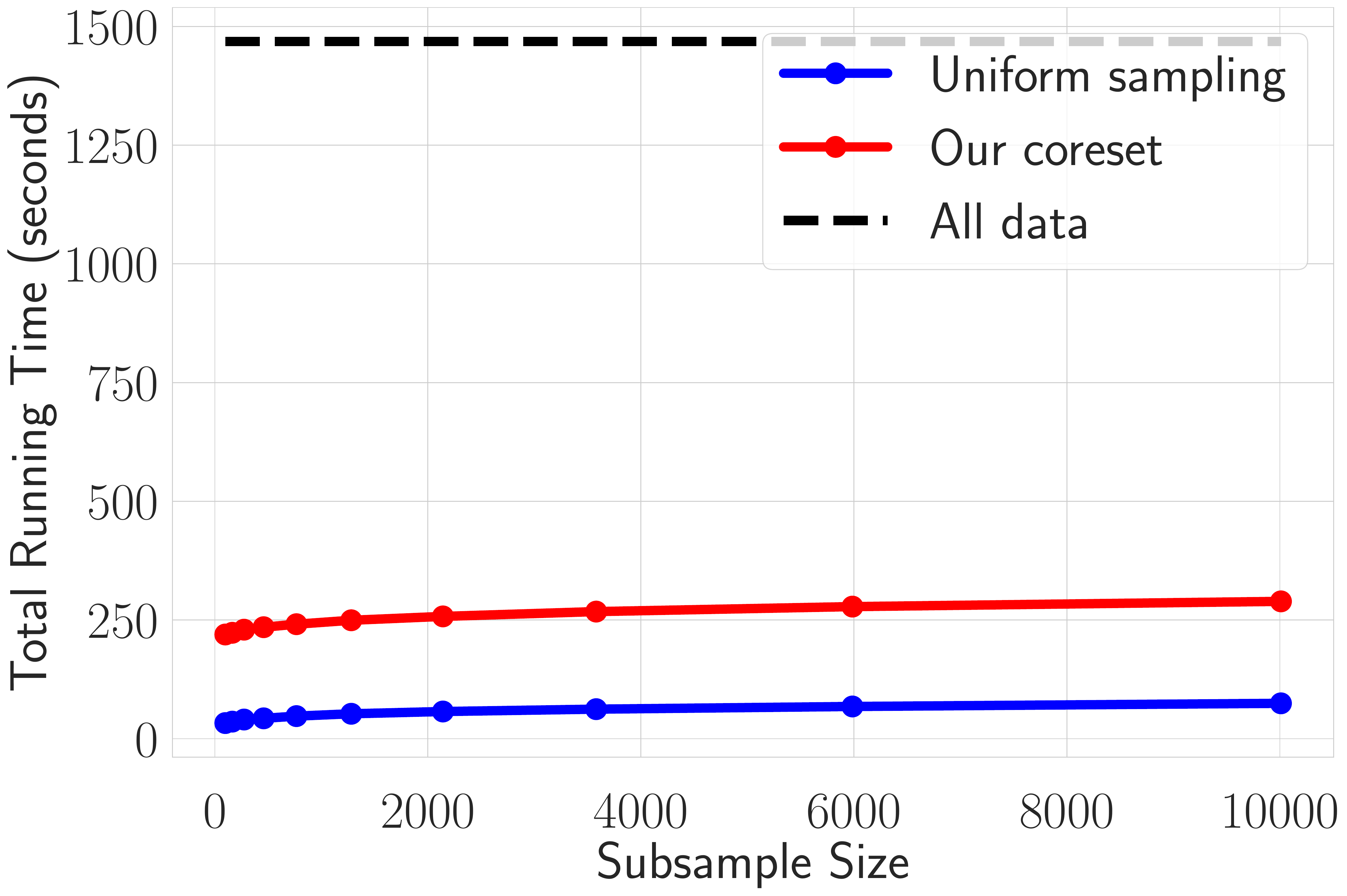}
\caption{}
\label{fig:reg_synth_huber}
\end{subfigure}
\begin{subfigure}[t]{0.49\textwidth}
\centering
\includegraphics[width=.49\textwidth]{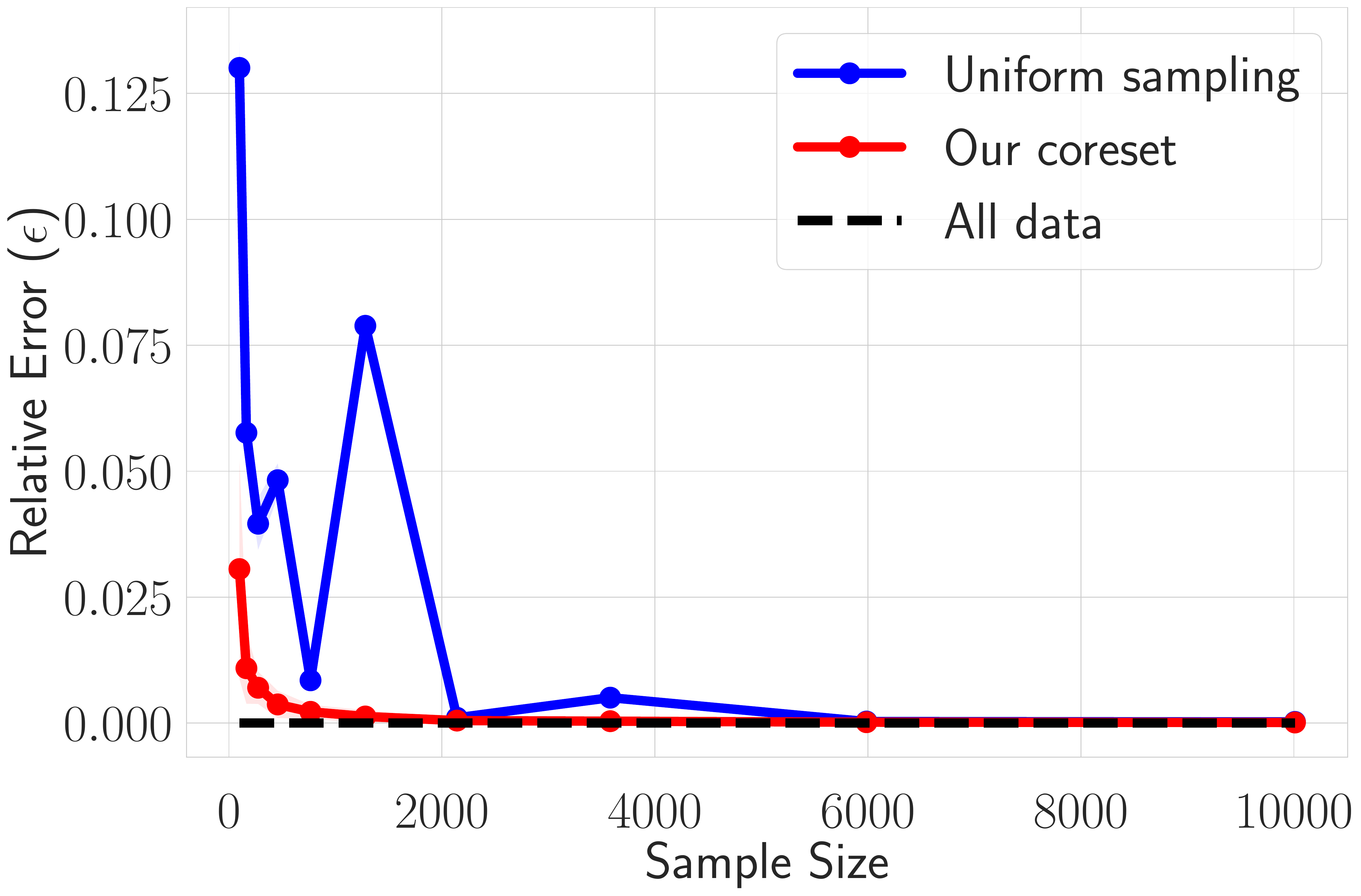}
\includegraphics[width=.49\textwidth]{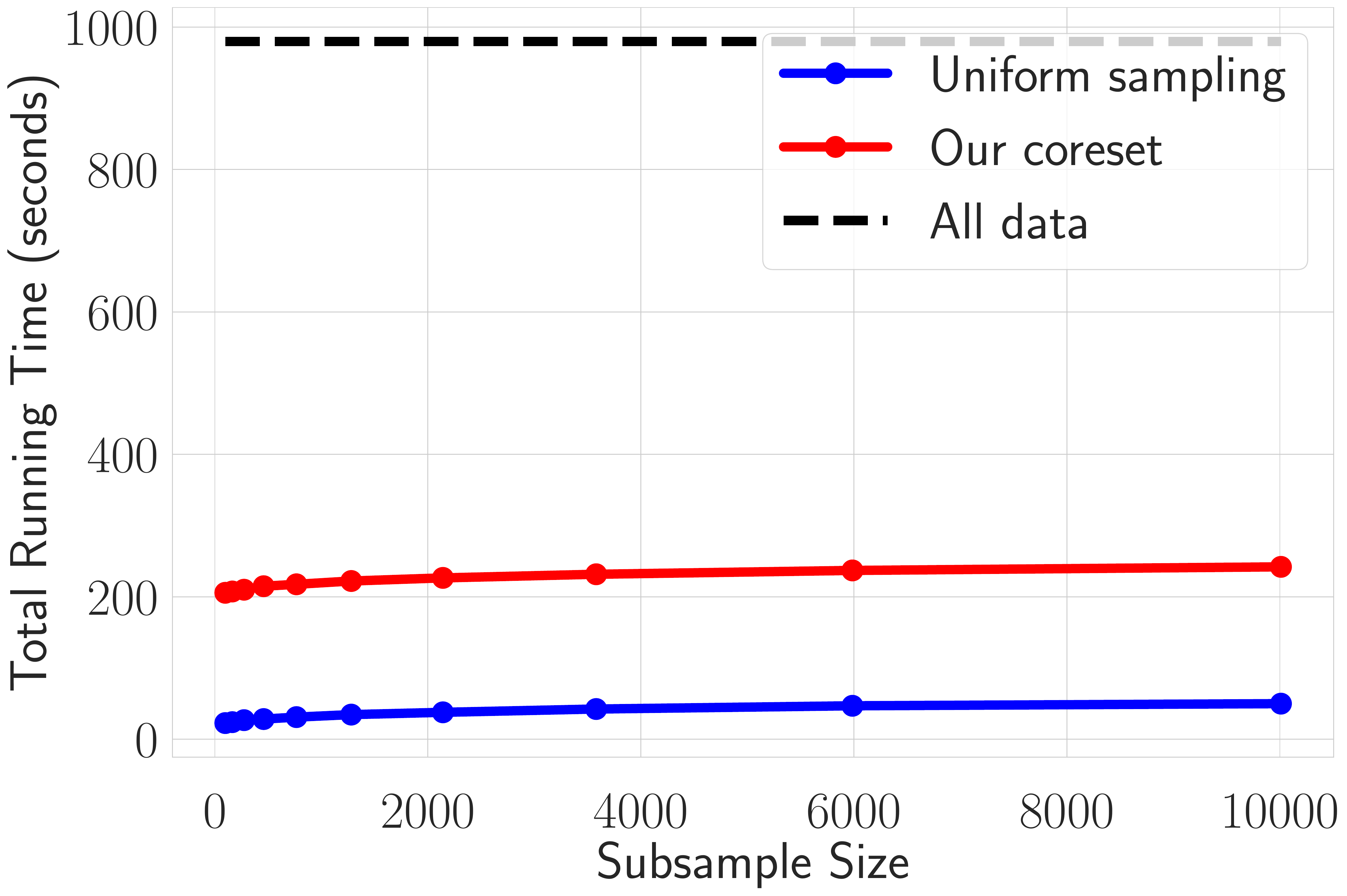}
\caption{}
\label{fig:reg_synth_cauchy}
\end{subfigure}
\begin{subfigure}[t]{.49\textwidth}
\centering
\includegraphics[width=.49\textwidth]{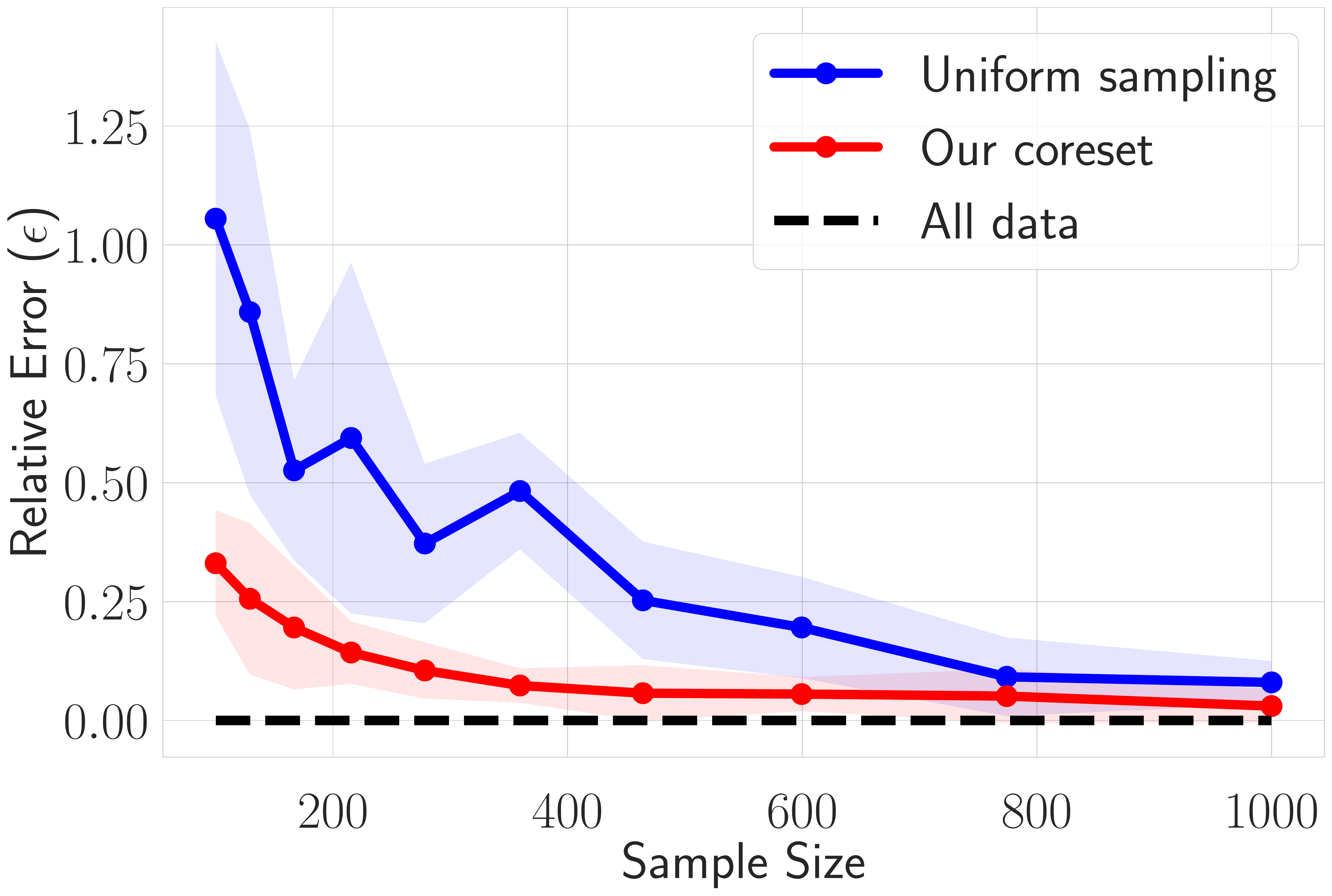}
\includegraphics[width=.49\textwidth]{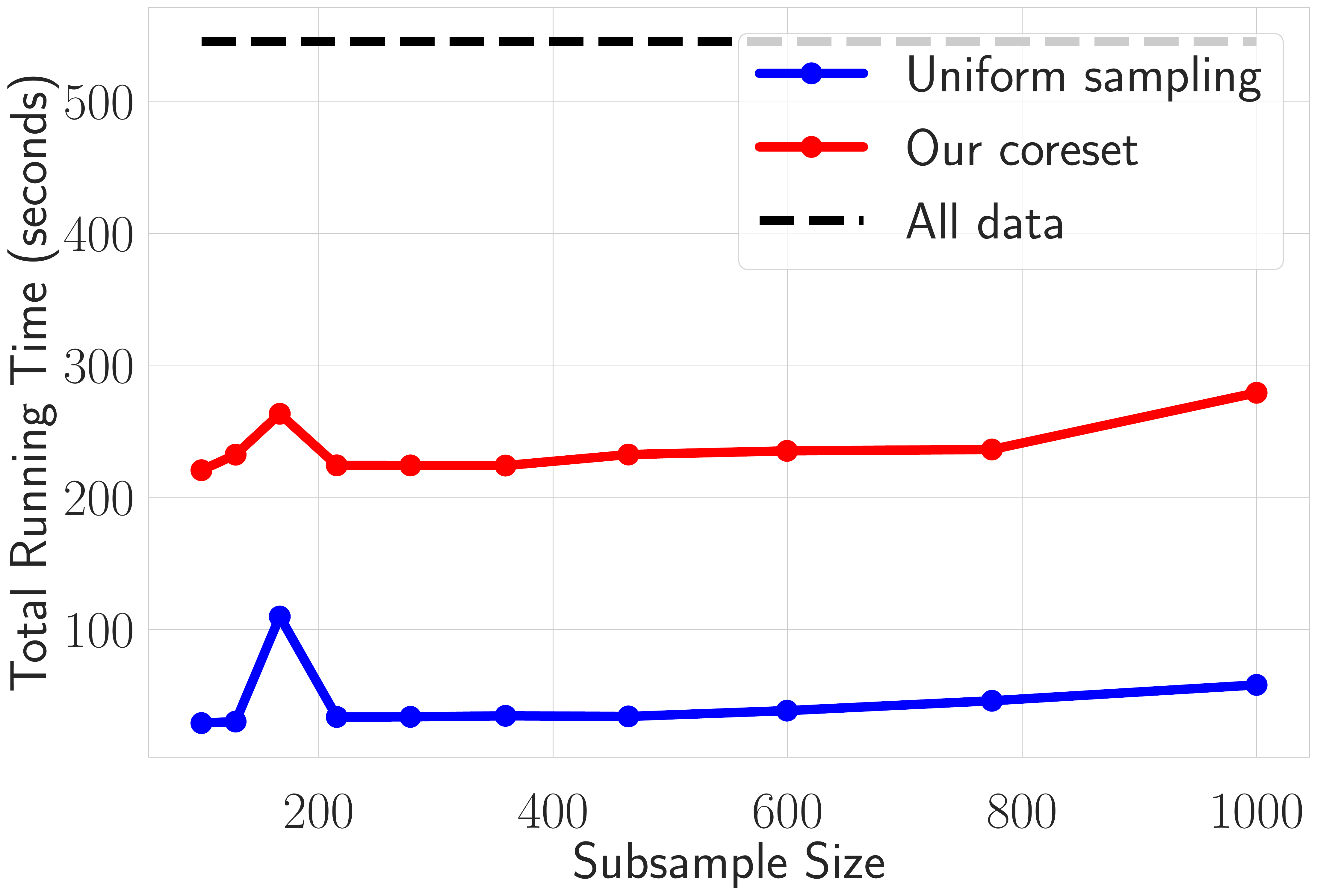}
\caption{}
\label{fig:proj_hour_2_2_l2}
\end{subfigure}
\begin{subfigure}[t]{.49\textwidth}
\centering
\includegraphics[width=.49\textwidth]{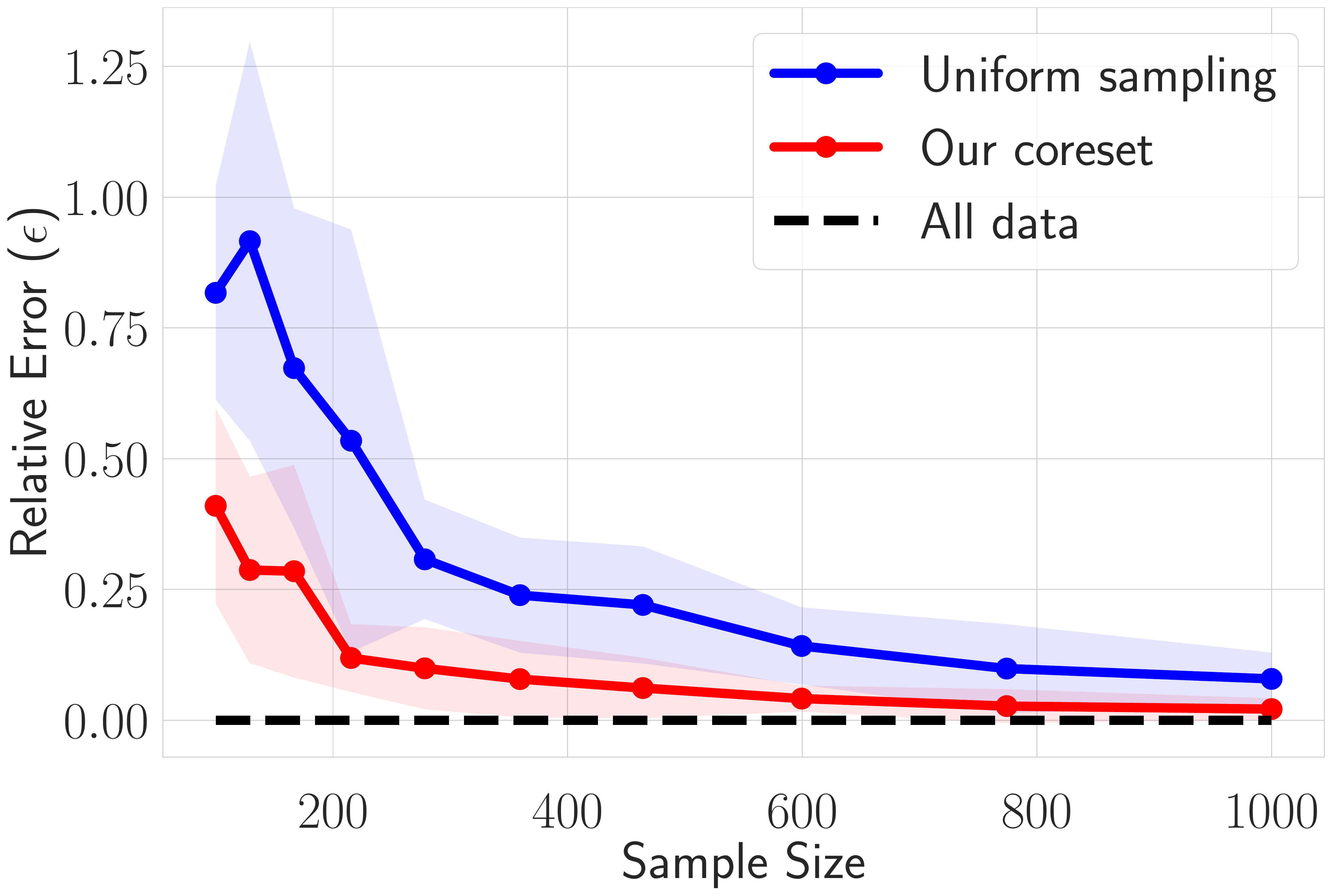}
\includegraphics[width=.49\textwidth]{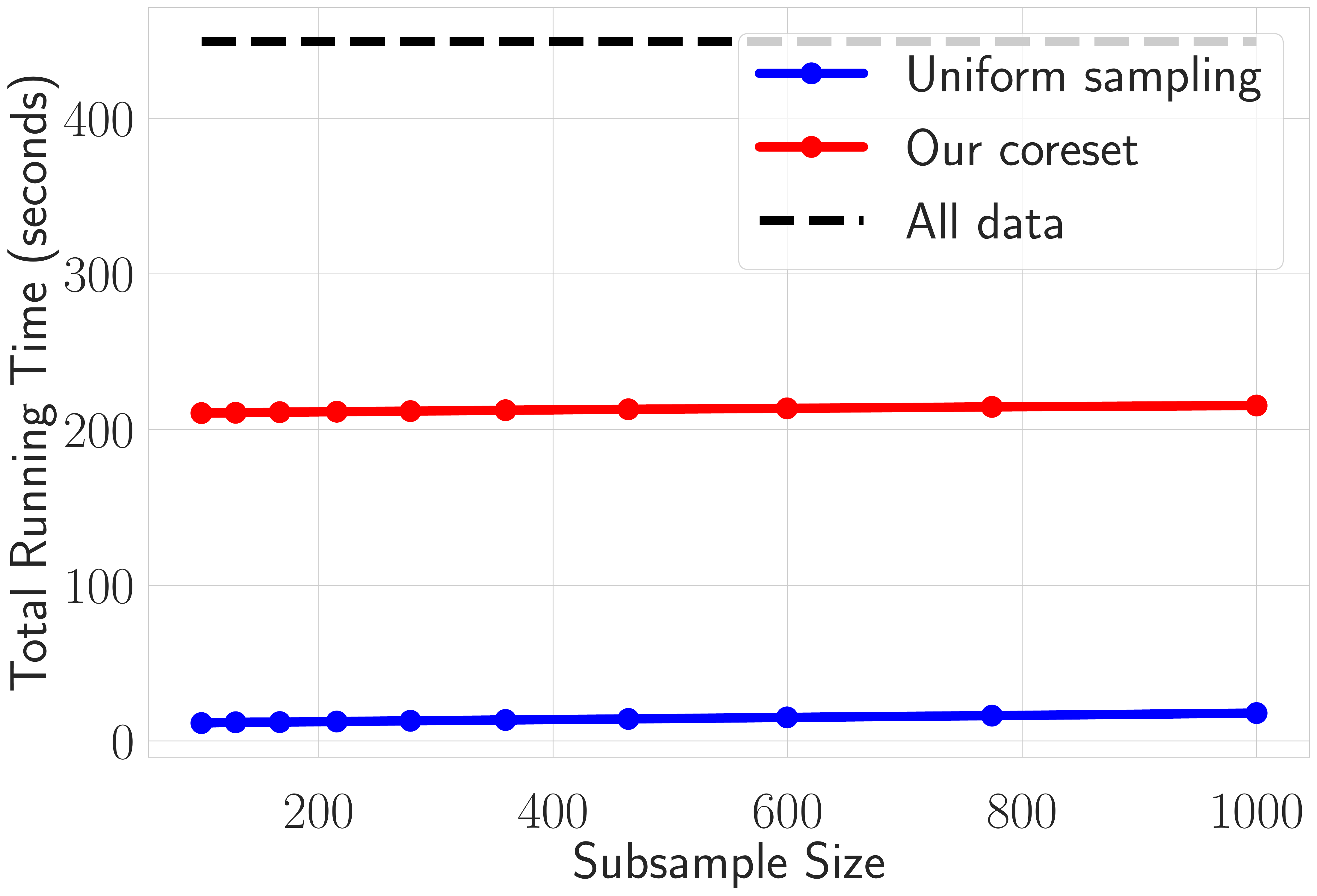}
\caption{}
\label{fig:proj_hour_2_2_Cauchy}
\end{subfigure}
\begin{subfigure}[t]{.49\textwidth}
\centering
\includegraphics[width=.49\textwidth]{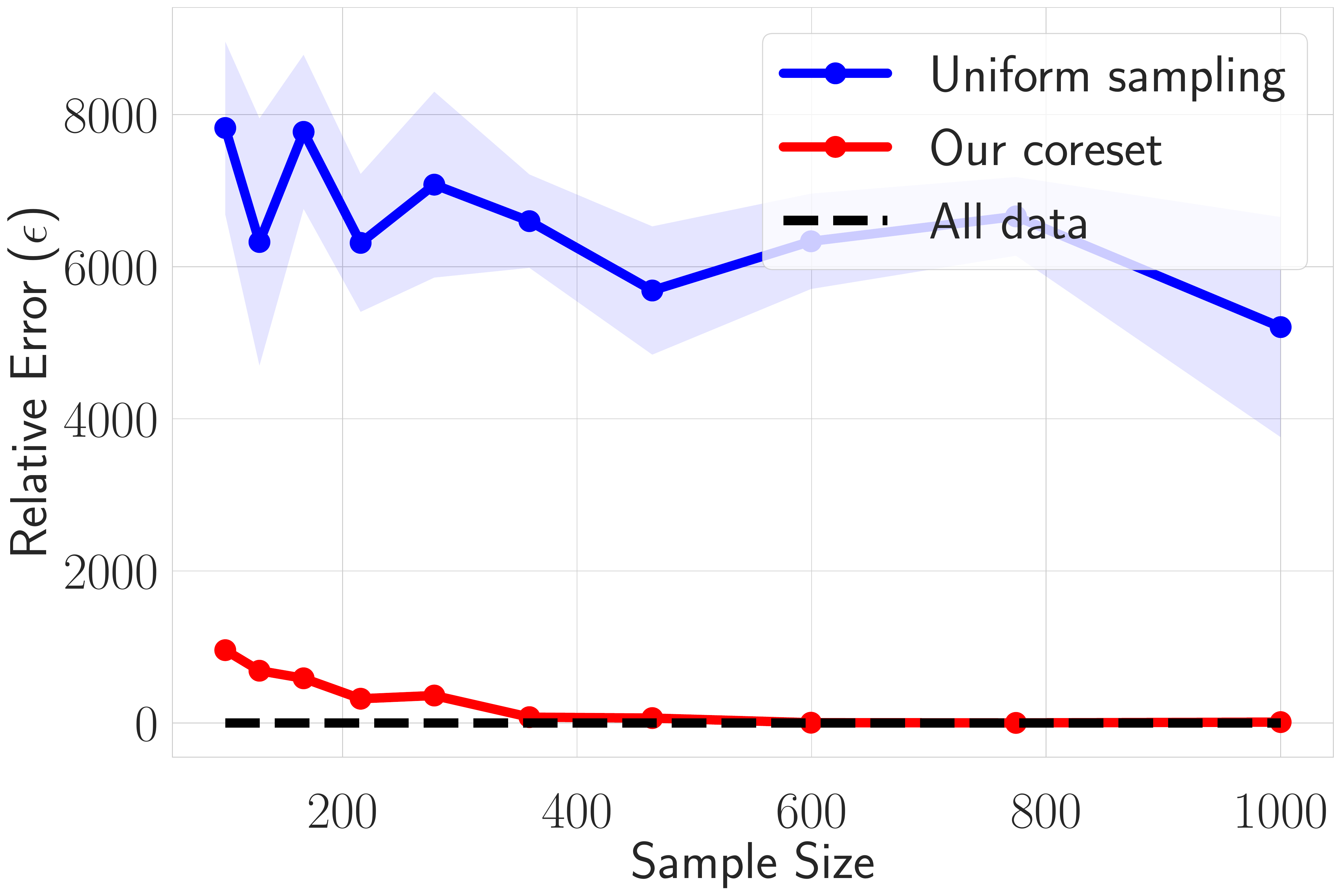}
\includegraphics[width=.49\textwidth]{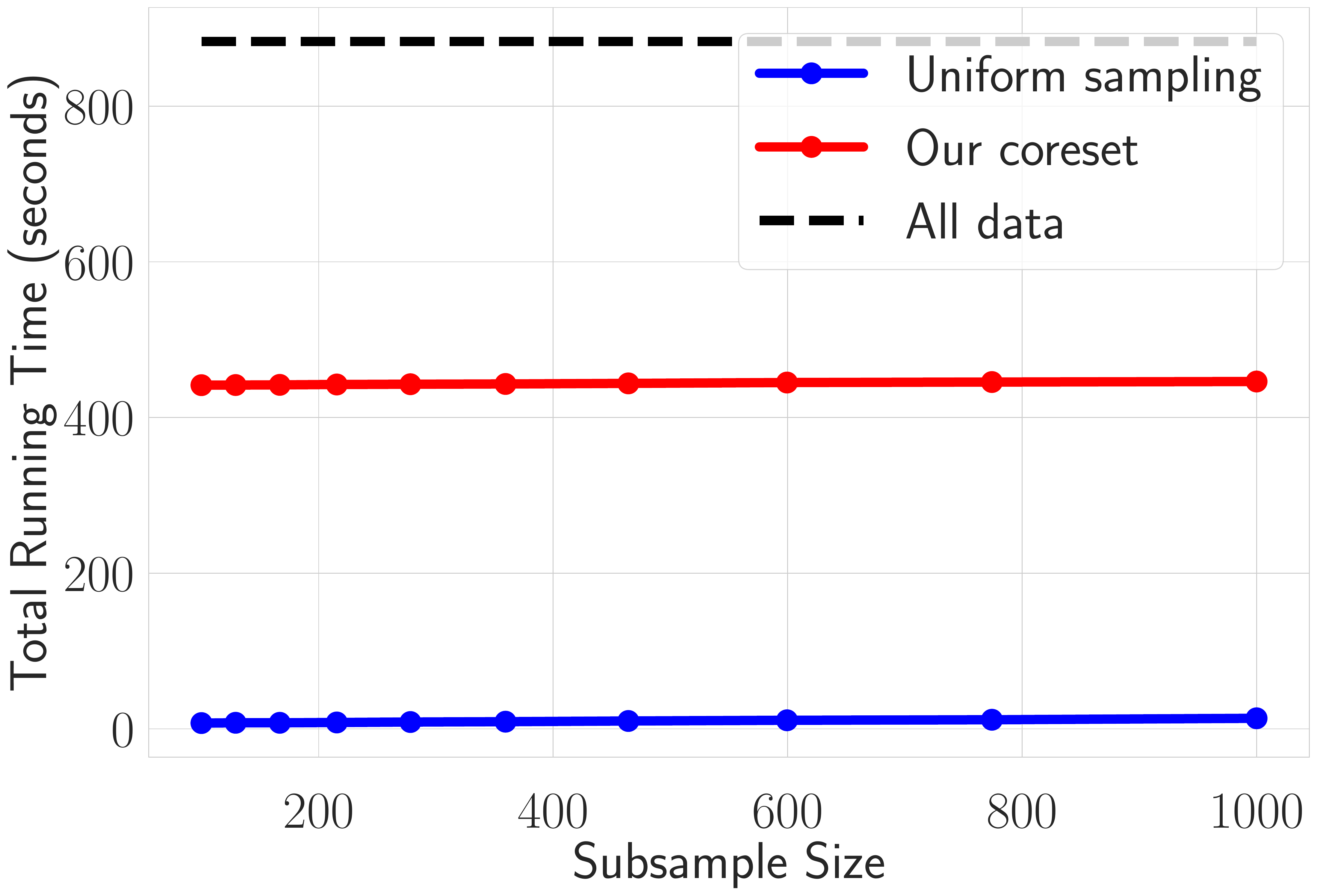}
\caption{}
\label{fig:proj_CASP_2_2_Tukey}
\end{subfigure}
\begin{subfigure}[t]{.49\textwidth}
\centering
\includegraphics[width=.49\textwidth]{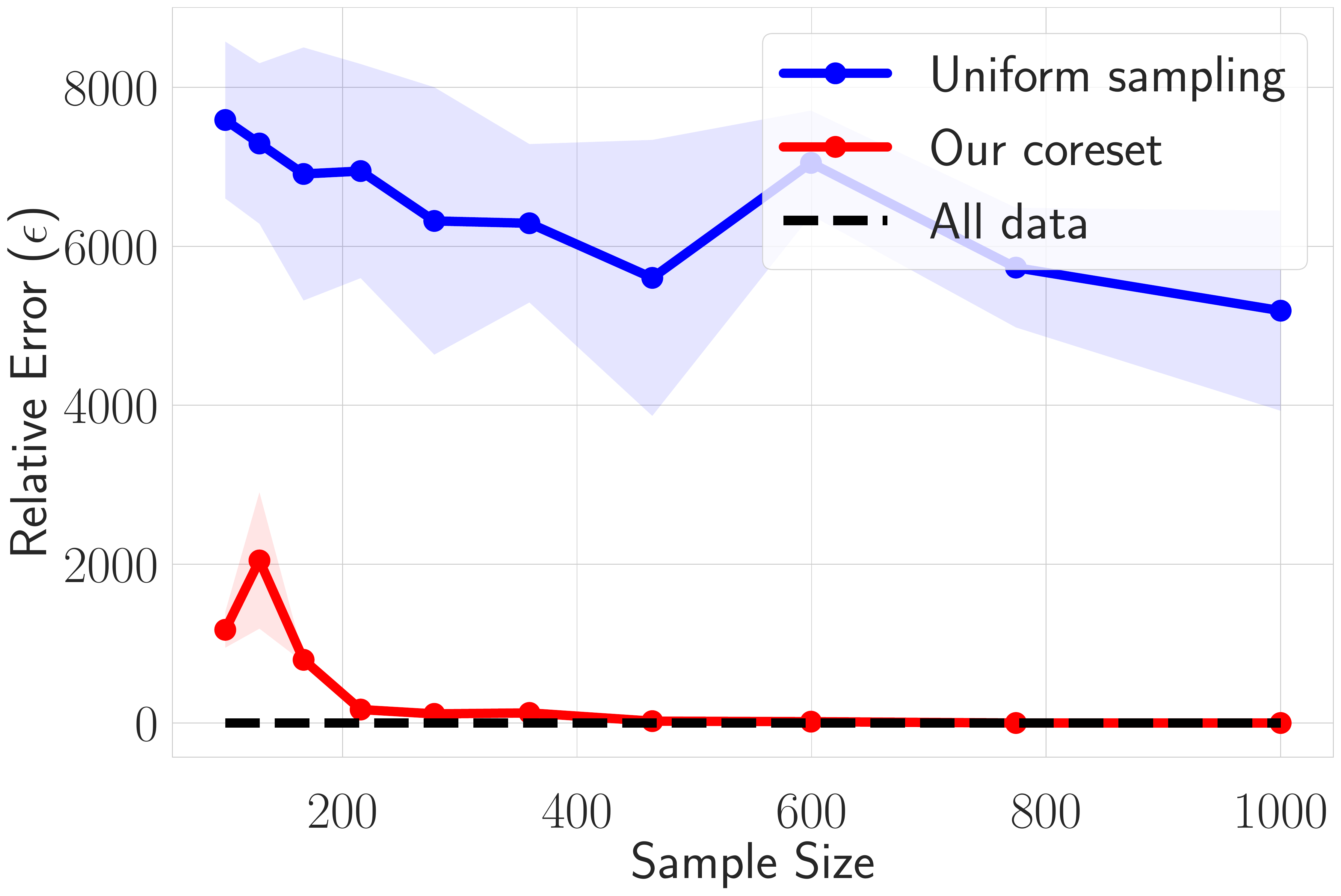}
\includegraphics[width=.49\textwidth]{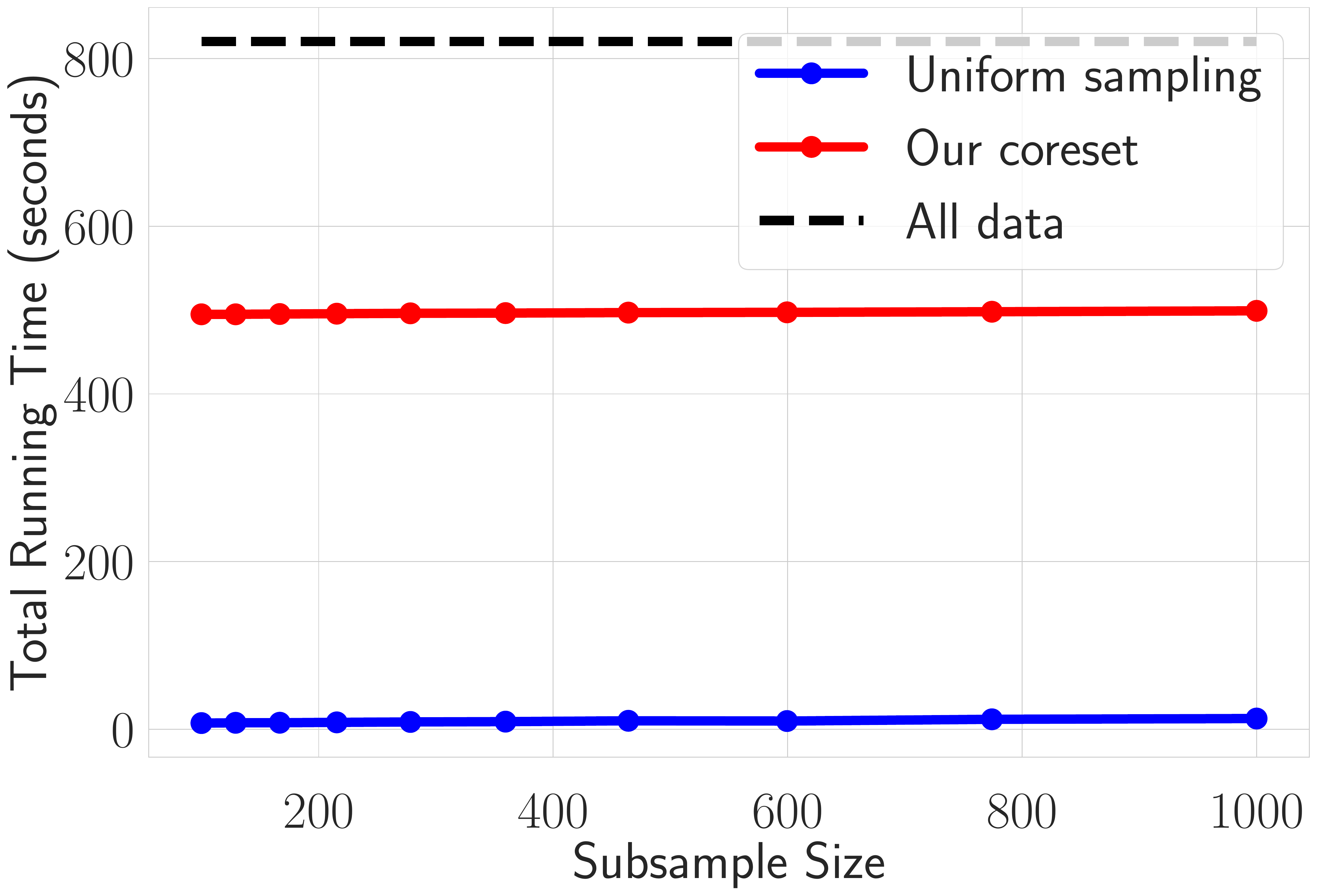}
\caption{}
\label{fig:proj_CASP_2_2_Welsch}
\end{subfigure}
\caption{Our experimental results: evaluating the efficacy of our coreset against uniform sampling.}
\label{fig:results}
\end{figure*}
In this section, we evaluate our coreset against uniform sampling on synthetic and real-world datasets, with respect to the projective clustering problem and its variants. 

\textbf{Software/Hardware.} Our algorithms were implemented~\cite{opencode} in Python 3.6~\citep{10.5555/1593511} using \say{Numpy}~\citep{oliphant2006guide}, \say{Scipy}~\citep{2020SciPy-NMeth}. Tests were performed on $2.59$GHz i$7$-$6500$U ($2$ cores total) machine with $16$GB RAM.

\textbf{Datasets.} The following datasets used for our experiments were mostly from UCI machine learning repository~\citep{Dua:2019}:
\begin{enumerate*}[label=(\roman*)]
    \item \label{dataset:1}\textbf{Synthetic} -- $20,000$ points in the two dimensional Euclidean space where $19,990$ points lie on the $x$-axis while the remaining $10$ points are generated away from the $x$-axis.
    \item \label{dataset:2}\textbf{Bike Sharing Dataset Data Set}~\citep{Dua:2019} -- consists of $17389$ samples, and $17$ features of which only $15$ were used for the sake of our comparisons. 
    \item \label{dataset:3}\textbf{Physicochemical Properties of Protein Tertiary Structure Data Set}~\citep{Dua:2019} -- $45,730$ samples, each consisting of $10$ features.
\end{enumerate*}

\textbf{Evaluation against uniform sampling.} Throughout the experiments, we have chosen $10$ sample sizes, starting from $100$ till $1,000$ for projective clustering problems and from $1,000$ till $10,000$ for regression problems; see Figure~\ref{fig:results}. At each sample size, we generate two coresets, where the first is using uniform sampling and the latter is using Algorithm~\ref{alg:full:projective}. When handling projective clustering problems, for each coreset $(S,v)$, we have computed a suboptimal solution $\tilde{H} \in \calH_j$ using an EM-like algorithm (Expectation Maximization) where the number of steps for convergence was $6$ while the number of different initializations was set to $1,000$. E.g., in Figure~\ref{fig:proj_hour_2_2_l2}, the goal was to find an suboptimal solution $\tilde{H}$ for the problem $\min_{H \in \calH_j} \sum_{p \in S} v(p) \mathrm{dist}\left( p, H\left( X,v \right)\right)^2$.
As for regression related problems, we have computed the suboptimal solution using Scipy's~\citep{2020SciPy-NMeth} own optimization sub-library which can handle such problem instances, where similarly to the projective clustering settings, we have ran the solver for $100$ iterations (at max) while having at max $15,000$ different initializations for the solver. The approximation error $\varepsilon$ is set to be the ratio $\sum_{p \in P} f\left(\mathrm{dist} \left(p, \tilde{H}\left( X, v\right)\right) \right)$ to $\left(\min_{H\in\calH_j} \sum_{p \in P} f\left(\mathrm{dist}\left( p, H(X,v)\right) \right) \right) - 1$. Finally, the results were averaged across $22$ trials, while the shaded regions correspond to the standard deviation.

\textbf{Choice of baseline.}
We remark that uniform sampling was selected as the baseline for our algorithm because the only existing coreset construction with theoretical guarantees for the integer $(j,k)$-projective clustering problem is that of \cite{EdwardsV05}. 
However, their construction is known to be impractical due to the large coreset size. 
In fact, even the base case requires a number of points that is exponential in $d$; thus we could not implement the coreset construction of \cite{EdwardsV05}. 
In practice uniform sampling is used due to the observation that real-world data is often not ``worst-case'' data. 
Thus it is a natural choice to compare the performance of our algorithm to that of uniform sampling across a number of real-world datasets, even though it is clear that we can generate synthetic data for which uniform sampling can perform arbitrarily badly due to its lack of provable guarantees, while our coreset constructions still maintains its theoretical guarantees.  

\textbf{Discussion.}
First note that our coresets are generally more accurate than uniform sampling across the experiments, sometimes outperforming uniform sampling by a factor of $\approx 10000$, e.g., $(2,2)$-projective clustering with the Tukey loss function in Figure~\ref{fig:proj_CASP_2_2_Tukey}. 
Moreover, there exist data distributions in which uniform sampling provably performs \emph{arbitrarily} worse than our coreset construction. 
For example, consider choosing $k=2$ centers across $n$ points when $n-1$ points are located at the origin and a single point is located at the position $N$ on the $x$-axis. 
Then the optimal clustering has cost zero by choosing a center at the origin and a center at $N$, but uniform sampling will not find the point at $N$ without $\Omega(N)$ samples and thus incur cost $N$. Since our coreset finds a multiplicative approximation to the optimal solution, it will also achieve a clustering with cost zero, which is arbitrarily better than $N$, sampling only $\polylog(n)$ points. On the other hand, in some of the figures, e.g., Figure~\ref{fig:proj_CASP_2_2_Tukey}, as we increase the sample size, the approximation error that corresponds to our coreset might increase at some sample sizes. This phenomenon is associated with the probabilistic nature of our coreset, as our coreset is a result of a sensitivity sampling technique. This problem can be easily resolved via increasing the number of trials (the number of trials was chosen to be $22$). The same holds for uniform sampling.

Although our coreset is generally better in terms of approximation error than uniform sampling, however the running time of our implementation is slow. We strongly believe that our algorithm can achieve faster results using the merge-and-reduce tree on the expense of an increase in the approximation error. For additional results, see Section~\ref{sec:exp_ext} at the appendix.

\section{CONCLUSIONS AND FUTURE WORK}
In this paper, we have provided an $L_\infty$ and $L_2$ coresets for $(k,j)$-projective clustering problems and its variants, e.g., $M$-estimators. Our approach leveraged an elegant combination between L\"{o}wner ellipsoid and Carath\'{e}odory's theorem. This in term sheds light on the use of constant-approximation coresets (our $L_\infty$ coreset) as a stepping stone towards $L_2$ coresets with $\varepsilon$ approximation. We believe that there is room for future work with respect to constructing $L_\infty$-coresets with smaller sizes for constant factor approximation. Finally, the lower bound on the size of constant factor coresets for the $(j,k)$-projective clustering problem is still unknown. We hope our work presents an important step in resolving the complexity of this problem.

\section{ACKNOWLEDGEMENTS}
This research was partially supported by the Israel National Cyber Directorate via the BIU Center for Applied Research in Cyber Security, and supported in part by NSF CAREER grant 1652257, NSF grant 1934979, ONR Award N00014-18-1-2364 and the Lifelong Learning Machines program from DARPA/MTO. In addition, Samson Zhou would like to thank National Institute of Health grant 5401 HG 10798-2 and a Simons Investigator Award of David P. Woodruff.

\bibliographystyle{apalike}
\bibliography{main}

\clearpage
\appendix

\thispagestyle{empty}

\section{APPLICATIONS}
\label{supplement:app}
In what follows, we will show that $\ell_\infty$-coreset can serve a family of functions including (but not bounded to) $M$-estimators.

\subsection{$L_\infty$ Coreset for Cauchy Regression}
\begin{lemma}
\label{lem:cauchy}
Let $P\subseteq\mathbb{R}^d$ be a set of $n$ points, $b:P\to\mathbb{R}$, $\lambda\in\mathbb{R}$, and let $\Psi_{Cau}$ denote the Cauchy loss function.
Let $P'=\{\p\circ b(\p)\,|\,\p\in P\}$, where $\circ$ denotes vertical concatenation.
Let $C'$ be the output of a call to $L_\infty-\coreset(P',d)$ and let $C\subseteq P$ so that $C'=\{\q\circ b(\q)\,|\,\q\in C\}$.
Then for every $\w\in\mathbb{R}^d$,
$\max_{\p\in P}\Psi_{Cau} \left(|\p^\top\w-b(\p)|\right)\le 8(d+1)^3\cdot\max_{\q\in C}\Psi_{Cau}\left(|\q^\top\w-b(\q)|\right).$
\end{lemma}

\begin{proof}
We first observe that the claim is trivially true for $\w=0^d$.
Thus it suffices to consider nonzero $\w\in\mathbb{R}^d$.
Let $Y\in\calH_{d-1}$ such that $\w^\top\Y=0^{d-1}$ and $\Y^\top\w=0^d$.
For each $\p\in P$, let $\p'=\p\circ b(\p)=\begin{bmatrix}\p\\ b(\p)\end{bmatrix}$ denote the vertical concatenation of $\p$ with $b(\p)$.
We also define the vertical concatenation $\w'=\w\circ(-1)=\begin{bmatrix}\w\\ -1\end{bmatrix}$.
By running $\coreset$ on $P'=\{\p'\,|\,\p\in P\}$ to obtain a coreset $C$, then we have by Theorem~\ref{thm:main:apps},
\begin{align*}
\max_{\p\in P}&\dist(\p',H(\w',0^{d+1}))^z\le2^{z+1}(d+1)^{1.5z}\max_{\q\in C}\dist(\q',H(\w',0^{d+1}))^z.
\end{align*}
Thus for $z=2$, we have for every $\p\in P$,
\[|(\p')^\top\w'|^2\le8(d+1)^3\max_{\q\in C}|(\q')^\top\w'|^2.\]
The Cauchy loss function is monotonically increasing, so that $\Psi_{Cau}(|\p^\top\x-b(\p)|)$ increases as $|\p^\top\x-b(\p)|$ increases.
Thus for every $\p\in P$,
\begin{align*}
\Psi_{Cau}(|\p^\top\w-b(\p)|)&=\Psi_{Cau}\left(\left|(\p')^\top\w'\right|\right)\\
&=\frac{\lambda^2}{2}\log\left(1+\left(\frac{\left|(\p')^\top\w'\right|}{\lambda}\right)^2\right)\\
&\le\max_{\q\in C}\frac{\lambda^2}{2}\log\left(1+8(d+1)^3\left(\frac{\left|(\q')^\top\w'\right|}{\lambda}\right)^2\right),
\end{align*}
where the inequality follows from the $L_\infty$-coreset property above and the monotonicity of the Cauchy loss function.
Thus by Bernoulli's inequality, we have
\begin{align*}
\Psi_{Cau}(|\p^\top\w-b(\p)|)&\le\max_{\q\in C}8(d+1)^3\cdot\frac{\lambda^2}{2}\log\left(1+\left(\frac{\left|(\q')^\top\w'\right|}{\lambda}\right)^2\right)\\
&=8(d+1)^3\max_{\q\in C}\Psi_{Cau}\left(\left|(\q')^\top\w'\right|\right)\\
&=8(d+1)^3\max_{\q\in C}\Psi_{Cau}\left(\q^\top\w-b(\q)\right).
\end{align*}
\end{proof}

\subsection{$L_\infty$ Coreset for Welsch Regression}
First, we will present the following as a stepping stone towards bounding the approximation error that our $L_\infty$-coreset achieves in the context of Welsch regression problem.
\begin{lemma}
\label{lem:welsch:struct}
Let $a\ge 1$ be a positive real number.
Then for every $x\in\mathbb{R}$,
\[1-e^{-a^2x^2}\le a^2(1-e^{-x^2}).\]
\end{lemma}
\begin{proof}
Since $e^{-x^2}$ decreases as $x^2$ increases, then $a^2e^{-x^2}-e^{-a^2x^2}$ is a monotonically non-increasing function that achieves its maximum at $x=0$.
In particular, the value of $a^2e^{-x^2}-e^{-a^2x^2}$ at $x=0$ is $a^2-1$, so that
\[a^2e^{-x^2}-e^{-a^2x^2}\le a^2-1.\]
Thus from rearranging the terms, we have that
\[1-e^{-a^2x^2}\le a^2(1-e^{-x^2}).\]
\end{proof}

\begin{lemma}
\label{lem:welsch}
Let $P\subseteq\mathbb{R}^d$ be a set of $n$ points, $b:P\to\mathbb{R}$, $\lambda\in\mathbb{R}$, and let $\Psi_{Wel}$ denote the Welsch loss function.
Let $P'=\{\p\circ b(\p)\,|\,\p\in P\}$, where $\circ$ denotes vertical concatenation.
Let $C'$ be the output of a call to $L_\infty-\coreset(P',d)$ and let $C\subseteq P$ so that $C'=\{\q\circ b(\q)\,|\,\q\in C\}$.
Then for every $\w\in\mathbb{R}^d$,
$\max_{\p\in P}\Psi_{Wel} \left(|\p^\top\w-b(\p)|\right)\le 8(d+1)^3\cdot\max_{\q\in C}\Psi_{Wel}\left(|\q^\top\w-b(\q)|\right).$
\end{lemma}

\begin{proof}
We observe that the claim is trivially true for $\w=0^d$, so that it suffices to consider nonzero $\w\in\mathbb{R}^d$.
Let $Y\in\calH_{d-1}$, so that $\w^\top\Y=0^{d-1}$ and $\Y^\top\w=0^d$, and for each $\p\in P$, let $\p'=\p\circ b(\p)=\begin{bmatrix}\p\\ b(\p)\end{bmatrix}$ denote the vertical concatenation of $\p$ with $b(\p)$.
Let $\w'$ denote the vertical concatenation $\w'=\w\circ(-1)=\begin{bmatrix}\w\\ -1\end{bmatrix}$.
By Theorem~\ref{thm:main:apps}, we have that the output $C$ of $\coreset$ on $P'=\{\p'\,|\,\p\in P\}$ satisfies
\begin{align*}
\max_{\p\in P}&\dist(\p',H(\w',0^{d+1}))^z\le2^{z+1}(d+1)^{1.5z}\max_{\q\in C}\dist(\q',H(\w',0^{d+1}))^z.
\end{align*}
Thus for $z=2$, we have for every $\p\in P$,
\[|(\p')^\top\w'|^2\le8(d+1)^3\max_{\q\in C}|(\q')^\top\w'|^2.\]
The Welsch loss function is monotonically increasing, so that $\Psi_{Wel}(|\p^\top\x-b(\p)|)$ increases as $|\p^\top\x-b(\p)|$ increases.
Hence, for every $\p\in P$,
\begin{align*}
\Psi_{Wel}(|\p^\top\w-b(\p)|)&=\Psi_{Wel}\left(\left|(\p')^\top\w'\right|\right)\\
&=\frac{\lambda^2}{2}\left(1-e^{-\left(\frac{\left|(\p')^\top\w'\right|}{\lambda}\right)^2}\right)\\
&\le\max_{\q\in C}\frac{\lambda^2}{2}\left(1-e^{-\left(\frac{8(d+1)^3\left|(\q')^\top\w'\right|}{\lambda}\right)^2}\right),
\end{align*}
where the inequality results from the $L_\infty$-coreset property above and the monotonicity of the Welsch loss function.
By Lemma~\ref{lem:welsch:struct},
\begin{align*}
\Psi_{Wel}(|\p^\top\w-b(\p)|)&\le\max_{\q\in C}8(d+1)^3\cdot\frac{\lambda^2}{2}\left(1-e^{-\left(\frac{\left|(\q')^\top\w'\right|}{\lambda}\right)^2}\right)\\
&=8(d+1)^3\max_{\q\in C}\Psi_{Wel}\left(\left|(\q')^\top\w'\right|\right)\\
&=8(d+1)^3\max_{\q\in C}\Psi_{Wel}\left(|\q^\top\w-b(\q)|\right).
\end{align*}
\end{proof}

\subsection{$L_\infty$ coreset for Huber regression}
\begin{lemma}
\label{lem:huber}
Let $P\subseteq\mathbb{R}^d$ be a set of $n$ points, $b:P\to\mathbb{R}$, $\lambda\in\mathbb{R}$, and let $\Psi_{Hub}$ denote the Huber loss function.
Let $P'=\{\p\circ b(\p)\,|\,\p\in P\}$, where $\circ$ denotes vertical concatenation.
Let $C'$ be the output of a call to $L_\infty-\coreset(P',d)$ and let $C\subseteq P$ so that $C'=\{\q\circ b(\q)\,|\,\q\in C\}$.
Then for every $\w\in\mathbb{R}^d$,
$\max_{\p\in P}\Psi_{Hub} \left(|\p^\top\w-b(\p)|\right)\le 16(d+1)^3\cdot\max_{\q\in C}\Psi_{Hub}\left(|\q^\top\w-b(\q)|\right).$
\end{lemma}

\begin{proof}
The claim is trivially true for $\w=0^d$; it remains to consider nonzero $\w\in\mathbb{R}^d$.
Let $Y\in\calH_{d-1}$, so that $\w^\top\Y=0^{d-1}$ and $\Y^\top\w=0^d$.
For each $\p\in P$, we use $\p'$ to denote the vertical concatenation of $\p$ with $b(\p)$, $\p':=\p\circ b(\p)=\begin{bmatrix}\p\\ b(\p)\end{bmatrix}$.
Similarly, we use $\w'$ to denote the vertical concatenation $\w'=\w\circ(-1)=\begin{bmatrix}\w\\ -1\end{bmatrix}$.
By Theorem~\ref{thm:main:apps}, we have that the output $C$ of $\coreset$ on $P'=\{\p'\,|\,\p\in P\}$ satisfies
\begin{align*}
\max_{\p\in P}\dist(\p',H(\w',0^{d+1}))^z&\le2^{z+1}(d+1)^{1.5z}\max_{\q\in C}\dist(\q',H(\w',0^{d+1}))^z.
\end{align*}
Thus for $z=2$, we have for every $\p\in P$,
\begin{align}
\label{eqn:huber:base}
|(\p')^\top\w'|^2\le8(d+1)^3\max_{\q\in C}|(\q')^\top\w'|^2.
\end{align}
We now consider casework for whether $|(\p')^\top\w'|\le\lambda$ or $|(\p')^\top\w'|>\lambda$.

If $|(\p')^\top\w'|\le\lambda$, then we immediately have from (\ref{eqn:huber:base}) and the fact that $C\subseteq P$ that
\[\Psi_{Hub}\left(|(\p')^\top\w'|\right)\le 8(d+1)^3\max_{\q\in C}\Psi_{Hub}\left(|(\q')^\top\w'|\right).\]
On the other hand if $|(\p')^\top\w'|>\lambda$, we further consider casework for whether $\max_{\q\in C}|(\q')^\top\w'|\le\lambda$ or $\max_{\q\in C}|(\q')^\top\w'|>\lambda$.
If $\max_{\q\in C}|(\q')^\top\w'|>\lambda$, then we again have from (\ref{eqn:huber:base}) and the fact that $C\subseteq P$ that
\[\Psi_{Hub}\left(|(\p')^\top\w'|\right)\le 8(d+1)^3\max_{\q\in C}\Psi_{Hub}\left(|(\q')^\top\w'|\right).\]
Finally, if $|(\p')^\top\w'|>\lambda$ but $\max_{\q\in C}|(\q')^\top\w'|\le\lambda$, then we observe that from (\ref{eqn:huber:base}) and the assumption that $|(\p')^\top\w'|>\lambda$, we have
\[\frac{\lambda}{\sqrt{8}(d+1)^{1.5}}\le\max_{\q\in C}|(\q')^\top\w'|.\]
Thus if $|(\p')^\top\w'|>\lambda$, then
\begin{align*}
\Psi_{Hub}(|\p^\top\w-b(\p)|)&=\Psi_{Hub}\left(\left|(\p')^\top\w'\right|\right)\\
&=\lambda\left(\left|(\p')^\top\w'\right|-\frac{\lambda}{2}\right)\\
&\le\lambda\left(\left|(\p')^\top\w'\right|\right)\\
&\le\sqrt{8}\lambda(d+1)^{1.5}\left(\max_{\q\in C}\left|(\q')^\top\w'\right|\right),
\end{align*}
where the last inequality results from the $L_\infty$-coreset property in (\ref{eqn:huber:base}) above.
Therefore,
\begin{align*}
\Psi_{Hub}(|\p^\top\w-b(\p)|)&\le\frac{\lambda}{\sqrt{8}(d+1)^{1.5}}\cdot8(d+1)^3\left(\max_{\q\in C}\left|(\q')^\top\w'\right|\right)\\
&\le8(d+1)^3\left(\max_{\q\in C}\left|(\q')^\top\w'\right|^2\right)\\
&\le16(d+1)^3\max_{\q\in C}\Psi_{Hub}(|\q^\top\w-b(\q)|).
\end{align*}
Thus in all cases, we have
\begin{align*}
\max_{\p\in P}&\Psi_{Hub}\left(|\p^\top\w-b(\p)|\right)\le16(d+1)^3\cdot\max_{\q\in C}\Psi_{Hub}\left(|\q^\top\w-b(\q)|\right).
\end{align*}
\end{proof}

\subsection{$L_\infty$ coreset for Geman-McClure regression}
\begin{lemma}
\label{lem:gm}
Let $P\subseteq\mathbb{R}^d$ be a set of $n$ points, $b:P\to\mathbb{R}$, $\lambda\in\mathbb{R}$, and let $\Psi_{GM}$ denote the Geman-McClure loss function.
Let $P'=\{\p\circ b(\p)\,|\,\p\in P\}$, where $\circ$ denotes vertical concatenation.
Let $C'$ be the output of a call to $L_\infty-\coreset(P',d)$ and let $C\subseteq P$ so that $C'=\{\q\circ b(\q)\,|\,\q\in C\}$.
Then for every $\w\in\mathbb{R}^d$,
$\max_{\p\in P}\Psi_{GM} \left(|\p^\top\w-b(\p)|\right)\le 8(d+1)^3\cdot\max_{\q\in C}\Psi_{GM}\left(|\q^\top\w-b(\q)|\right).$
\end{lemma}

\begin{proof}
Note that the claim is trivially true for $\w=0^d$, so it therefore suffices to consider nonzero $\w\in\mathbb{R}^d$. 
Let $Y\in\calH_{d-1}$ such that $\w^\top\Y=0^{d-1}$ and $\Y^\top\w=0^d$.
For each $\p\in P$, let $\p'=\p\circ b(\p)=\begin{bmatrix}\p\\ b(\p)\end{bmatrix}$ denote the vertical concatenation of $\p$ with $b(\p)$.
We also define the vertical concatenation $\w'=\w\circ(-1)=\begin{bmatrix}\w\\ -1\end{bmatrix}$.
By setting $C$ to be the output of $\coreset$ on $P'=\{\p'\,|\,\p\in P\}$, then by Theorem~\ref{thm:main:apps},
\begin{align*}
\max_{\p\in P}&\dist(\p',H(\w',0^{d+1}))^z\le2^{z+1}(d+1)^{1.5z}\max_{\q\in C}\dist(\q',H(\w',0^{d+1}))^z.
\end{align*}
Thus for $z=2$, we have for every $\p\in P$,
\[|(\p')^\top\w'|^2\le8(d+1)^3\max_{\q\in C}|(\q')^\top\w'|^2.\]
The Geman-McClure loss function is monotonically increasing, so that $\Psi_{GM}(|\p^\top\x-b(\p)|)$ increases as $|\p^\top\x-b(\p)|$ increases.
Therefore,
\begin{align*}
\max_{\p\in P}\Psi_{GM}\left(|\p^\top\w-b(\p)|\right)&=\max_{\p\in P}\frac{|(\p')^\top\w|^2}{2+2|(\p')^\top\w|^2)}\\
&\le\max_{\q\in C}\frac{8(d+1)^3|(\q')^\top\w|^2}{2+2|(\q')^\top\w|^2}\\
&=8(d+1)^3\max_{\q\in C}\Psi_{GM}\left(|\q^\top\w-b(\q)|\right),
\end{align*}
where the inequality results from the $L_\infty$-coreset property of Theorem~\ref{thm:main:apps} and the fact that $C\subseteq P$. 
\end{proof}

\subsection{$L_\infty$ Coreset for Regression with Concave Loss Function}
We first recall the following property of concave functions:
\begin{lemma}
\label{lem:concave:struct}
Let $f:\mathbb{R}\to\mathbb{R}$ be a concave function with $f(0)=0$. 
Then for any $x\le y$, we have $\frac{f(x)}{x}\ge\frac{f(y)}{y}$. 
\end{lemma}

Using Lemma~\ref{lem:concave:struct}, we obtain an $L_\infty$ coreset for regression for any non-decreasing concave loss function $\Psi_{Con}$ satisfying $\Psi_{Con}(0)=0$. 

We obtain an $L_\infty$ coreset for regression for any non-decreasing concave loss function $\Psi_{Con}$ satisfying $\Psi_{Con}(0)=0$.
\begin{lemma}
\label{lem:concave}
Let $P\subseteq\mathbb{R}^d$ be a set of $n$ points, $b:P\to\mathbb{R}$, $\lambda\in\mathbb{R}$, and let $\Psi_{Con}$ denote any non-decreasing concave loss function with $\Psi_{Con}(0)=0$. 
Let $P'=\{\p\circ b(\p)\,|\,\p\in P\}$, where $\circ$ denotes vertical concatenation.
Let $C'$ be the output of a call to $L_\infty-\coreset(P',d)$ and let $C\subseteq P$ so that $C'=\{\q\circ b(\q)\,|\,\q\in C\}$.
Then for every $\w\in\mathbb{R}^d$, $\max_{\p\in P}\Psi_{Con} \left(|\p^\top\w-b(\p)|\right)\le 4(d+1)^{1.5}\cdot\max_{\q\in C}\Psi_{Con}\left(|\q^\top\w-b(\q)|\right).$
\end{lemma}

\begin{proof}
The claim is trivially true for $\w=0^d$, so it suffices to consider nonzero $\w\in\mathbb{R}^d$. 
Let $Y\in\calH_{d-1}$ such that $\w^\top\Y=0^{d-1}$ and $\Y^\top\w=0^d$.
For each $\p\in P$, let $\p'=\p\circ b(\p)=\begin{bmatrix}\p\\ b(\p)\end{bmatrix}$ denote the vertical concatenation of $\p$ with $b(\p)$.
We also define the vertical concatenation $\w'=\w\circ(-1)=\begin{bmatrix}\w\\ -1\end{bmatrix}$.
By setting $C$ to be the output of $\coreset$ on $P'=\{\p'\,|\,\p\in P\}$, then by Theorem~\ref{thm:main:apps},
\begin{align*}
\max_{\p\in P}\dist(\p',H(\w',0^{d+1}))^z&\le2^{z+1}(d+1)^{1.5z}\max_{\q\in C}\dist(\q',H(\w',0^{d+1}))^z.
\end{align*}
Thus for $z=1$, we have for every $\p\in P$,
\[|(\p')^\top\w'|\le4(d+1)^{1.5}\max_{\q\in C}|(\q')^\top\w'|.\]
Since $\Psi_{Con}$ is monotonically non-decreasing, then $\Psi_{Con}(|\p^\top\x-b(\p))$ increases as $|\p^\top\x-b(\p)|$ increases. 
Thus by Lemma~\ref{lem:concave:struct}, 
\begin{align*}
\max_{\p\in P}\Psi_{Con}\left(|\p^\top\w-b(\p)|\right)&\le\max_{\q\in C}\Psi_{Con}\left(4(d+1)^{1.5}|\q^\top\w-b(\q)|\right)\\
&\le4(d+1)^{1.5}\max_{\q\in C}\Psi_{Con}\left(|\q^\top\w-b(\q)|\right).
\end{align*}
\end{proof}

\subsection{$L_\infty$ Coreset for Tukey Regression}
\begin{lemma}
\label{lem:tukey}
Let $P\subseteq\mathbb{R}^d$ be a set of $n$ points, $b:P\to\mathbb{R}$, $\lambda\in\mathbb{R}$, and let $\Psi_{Tuk}$ denote the Tukey loss function.
Let $P'=\{\p\circ b(\p)\,|\,\p\in P\}$, where $\circ$ denotes vertical concatenation.
Let $C'$ be the output of a call to $L_\infty-\coreset(P',d)$ and let $C\subseteq P$ so that $C'=\{\q\circ b(\q)\,|\,\q\in C\}$.
Then for every $\w\in\mathbb{R}^d$, $\max_{\p\in P}\Psi_{Tuk} \left(|\p^\top\w-b(\p)|\right)\le 8(d+1)^3\cdot\max_{\q\in C}\Psi_{Tuk}\left(|\q^\top\w-b(\q)|\right).$
\end{lemma}

\begin{proof}
We first observe that the claim is trivially true for $\w=0^d$, so that it suffices to consider nonzero $\w\in\mathbb{R}^d$. 
Let $Y\in\calH_{d-1}$ such that $\w^\top\Y=0^{d-1}$ and $\Y^\top\w=0^d$.
For each $\p\in P$, let $\p'=\p\circ b(\p)=\begin{bmatrix}\p\\ b(\p)\end{bmatrix}$ denote the vertical concatenation of $\p$ with $b(\p)$.
We define the vertical concatenation $\w'=\w\circ(-1)=\begin{bmatrix}\w\\ -1\end{bmatrix}$.
By setting $C$ to be the output of $\coreset$ on $P'=\{\p'\,|\,\p\in P\}$, then by Theorem~\ref{thm:main:apps},
\begin{align*}
\max_{\p\in P}&\dist(\p',H(\w',0^{d+1}))^z\le2^{z+1}(d+1)^{1.5z}\max_{\q\in C}\dist(\q',H(\w',0^{d+1}))^z.
\end{align*}
Thus for $z=2$, we have for every $\p\in P$,
\[|(\p')^\top\w'|^2\le8(d+1)^3\max_{\q\in C}|(\q')^\top\w'|^2.\]
We first note that if $\frac{|(\p')^\top\w'|}{\sqrt{8}(d+1)^{1.5}}\ge\lambda$, then we trivially have $\max_{\q\in C}|(\q')^\top\w'|^2\ge\lambda^2$ so that $\max_{\q\in C}\Psi_{Tuk}\left(|\q^\top\w-b(\q)|\right)=\frac{\lambda^2}{6}\ge\Psi_{Tuk}(x)$ for all $x$. 
Thus, we would have
\[\max_{\p\in P}\Psi_{Tuk}\left(|\p^\top\w-b(\p)|\right)\le\max_{\q\in C}\Psi_{Tuk}\left(|\q^\top\w-b(\q)|\right),\] 
as desired. 
Hence, we assume $\frac{|(\p')^\top\w'|}{\sqrt{8}(d+1)^{1.5}}<\lambda$ and consider casework for whether $|(\p')^\top\w'|\le\lambda$ or $|(\p')^\top\w'|>\lambda$.

If $|(\p')^\top\w'|\le\lambda$, then since the Tukey loss function is monotonically increasing, we have
\begin{align*}
\Psi_{Tuk}\left(|(\p')^\top\w'|\right)&=\frac{\lambda^2}{6}\left(1-\left(1-\frac{|(\p')^\top\w'|^2}{\lambda^2}\right)^3\right)\\
&\le\max_{\q\in C}\frac{\lambda^2}{6}\left(1-\left(1-\frac{8(d+1)^3|(\q')^\top\w'|^2}{\lambda^2}\right)^3\right)
\end{align*}
Unfortunately, the Tukey loss function is not concave, so we cannot directly apply Lemma~\ref{lem:concave:struct}. 
However, if we define the function $f(x):=\frac{\lambda^2}{6}\left(1-\left(1-\frac{x}{\lambda^2}\right)^3\right)$, then we have
\[\frac{d^2f}{dx^2}=\frac{x-\lambda^2}{\lambda^4},\]
which is non-positive for all $x\le\lambda^2$. 
Thus by Lemma~\ref{lem:concave:struct}, we have for all $0\le x\le y\le\lambda^2$ that $\frac{f(x)}{x}\ge\frac{f(y)}{y}$. 
Since $f(x^2)=\Psi_{Tuk}(x)$, then we have for all $0\le x\le y\le\lambda$ that $\frac{\Psi_{Tuk}(x)}{x^2}\ge\frac{\Psi_{Tuk}(y)}{y^2}$. 
Hence by the assumption that $\frac{|(\p')^\top\w'|}{\sqrt{8}(d+1)^{1.5}}<\lambda$, 
\begin{align*}
\Psi_{Tuk}\left(|(\p')^\top\w'|\right)&\le 8(d+1)^3\max_{\q\in C}\frac{\lambda^2}{6}\left(1-\left(1-\frac{|(\q')^\top\w'|^2}{\lambda^2}\right)^3\right)\\
&\le8(d+1)^3\max_{\q\in C}\Psi_{Tuk}\left(|(\q')^\top\w'|\right)\\
&=8(d+1)^3\max_{\q\in C}\Psi_{Tuk}\left(|\q^\top\w-b(\q)|\right).
\end{align*}
On the other hand, if $|(\p')^\top\w'|>\lambda$, then we further consider casework on whether $\max_{\q\in C}|(\q')^\top\w'|>\lambda$ or $\max_{\q\in C}|(\q')^\top\w'|\le\lambda$. 
If $\max_{\q\in C}|(\q')^\top\w'|>\lambda$, then we immediately have 
\begin{align*}
\Psi_{Tuk}\left(|(\p')^\top\w'|\right)&=\frac{\lambda^2}{6}=\max_{\q\in C}\Psi_{Tuk}\left(|(\q')^\top\w'|\right)\\
&=\max_{\q\in C}\Psi_{Tuk}\left(|\q^\top\w-b(\q)|\right).
\end{align*}
Otherwise, suppose $\max_{\q\in C}|(\q')^\top\w'|\le\lambda$. 
Note that $|(\p')^\top\w'|^2\le8(d+1)^3\max_{\q\in C}|(\q')^\top\w'|^2$ implies $\max_{\q\in C}|(\q')^\top\w'|>\frac{\lambda}{\sqrt{8}(d+1)^{1.5}}$. 
Since the Tukey loss function is monotonically increasing, then 
\[\max_{\q\in C}\Psi_{Tuk}\left(|(\q')^\top\w'|\right)\ge\Psi_{Tuk}\left(\frac{\lambda}{\sqrt{8}(d+1)^{1.5}}\right).\]
Because $\max_{\q\in C}|(\q')^\top\w'|\le\lambda$, then we can again apply the relationship $\frac{\Psi_{Tuk}(x)}{x^2}\ge\frac{\Psi_{Tuk}(y)}{y^2}$ for all $0\le x\le y\le\lambda$, so that
\[\max_{\q\in C}\Psi_{Tuk}\left(|(\q')^\top\w'|\right)\ge\frac{1}{8(d+1)^3}\Psi_{Tuk}(\lambda).\]
Hence,
\begin{align*}
\Psi_{Tuk}\left(|(\p')^\top\w'|\right)=\Psi_{Tuk}(\lambda)\le 8(d+1)^3\max_{\q\in C}\Psi_{Tuk}\left(|(\q')^\top\w'|\right).
\end{align*}
Therefore across all cases, we have
\begin{align*}
\max_{\p\in P}\Psi_{Tuk}\left(|\p^\top\w-b(\p)|\right)\le 8(d+1)^3\cdot\max_{\q\in C}\Psi_{Tuk}\left(|\q^\top\w-b(\q)|\right).
\end{align*}
\end{proof}

\subsection{$L_\infty$ Coreset for $L_1-L_2$ Regression}
\begin{lemma}
\label{lem:ll}
Let $P\subseteq\mathbb{R}^d$ be a set of $n$ points, $b:P\to\mathbb{R}$, $\lambda\in\mathbb{R}$, and let $\Psi_{LL}$ denote the $L_1-L_2$ loss function.
Let $P'=\{\p\circ b(\p)\,|\,\p\in P\}$, where $\circ$ denotes vertical concatenation.
Let $C'$ be the output of a call to $L_\infty-\coreset(P',d)$ and let $C\subseteq P$ so that $C'=\{\q\circ b(\q)\,|\,\q\in C\}$.
Then for every $\w\in\mathbb{R}^d$, $\max_{\p\in P}\Psi_{LL} \left(|\p^\top\w-b(\p)|\right)\le 8(d+1)^3\cdot\max_{\q\in C}\Psi_{LL}\left(|\q^\top\w-b(\q)|\right).$
\end{lemma}

\begin{proof}
We first observe that the claim is trivially true for $\w=0^d$. 
Therefore, it suffices to consider nonzero $\w\in\mathbb{R}^d$. 
Let $Y\in\calH_{d-1}$ such that $\w^\top\Y=0^{d-1}$ and $\Y^\top\w=0^d$.
For each $\p\in P$, let $\p'=\p\circ b(\p)=\begin{bmatrix}\p\\ b(\p)\end{bmatrix}$ denote the vertical concatenation of $\p$ with $b(\p)$.
We also define the vertical concatenation $\w'=\w\circ(-1)=\begin{bmatrix}\w\\ -1\end{bmatrix}$.
By setting $C$ to be the output of $\coreset$ on $P'=\{\p'\,|\,\p\in P\}$, then by Theorem~\ref{thm:main:apps},
\begin{align*}
\max_{\p\in P}&\dist(\p',H(\w',0^{d+1}))^z\le2^{z+1}(d+1)^{1.5z}\max_{\q\in C}\dist(\q',H(\w',0^{d+1}))^z.
\end{align*}
Thus for $z=2$, we have for every $\p\in P$,
\[|(\p')^\top\w'|^2\le8(d+1)^3\max_{\q\in C}|(\q')^\top\w'|^2.\]
The $L_1-L_2$ loss function is monotonically increasing, so that $\Psi_{LL}(|\p^\top\x-b(\p))$ increases as $|\p^\top\x-b(\p)|$ increases.
Therefore,
\begin{align*}
\max_{\p\in P}\Psi_{LL}\left(|\p^\top\w-b(\p)|\right)&=\max_{\p\in P}2\left(\sqrt{1+\frac{|\p^\top\w-b(\p)|^2}{2}}-1\right)\\
&\le\max_{\q\in C}2\left(\sqrt{1+\frac{8(d+1)^3|\q^\top\w-b(\q)|^2}{2}}-1\right).
\end{align*}
Since the $L_1-L_2$ loss function is not concave, so we cannot directly apply Lemma~\ref{lem:concave:struct}. 
Fortunately, if we define the function $f(x):=2\left(\sqrt{1+\frac{x}{2}}-1\right)$, then we have
\[\frac{d^2f}{dx^2}=-\frac{2}{16\left(\frac{x}{2}+1\right)^{3/2}},\]
which is non-positive for all $x\ge 0$. 
Thus by Lemma~\ref{lem:concave:struct}, we have for all $0\le x\le y$ that $\frac{f(x)}{x}\ge\frac{f(y)}{y}$. 
Since $f(x^2)=\Psi_{LL}(x)$, then we have for all $0\le x\le y$ that $\frac{\Psi_{LL}(x)}{x^2}\ge\frac{\Psi_{LL}(y)}{y^2}$. 
Thus,
\begin{align*}
\max_{\p\in P}\Psi_{LL}\left(|\p^\top\w-b(\p)|\right)&\le 8(d+1)^3\max_{\q\in C}2\left(\sqrt{1+\frac{|\q^\top\w-b(\q)|^2}{2}}-1\right)\\
&=8(d+1)^3\max_{\q\in C}\Psi_{LL}\left(|\q^\top\w-b(\q)|\right).
\end{align*}
\end{proof}

\subsection{$L_\infty$ Coreset for Fair Regression}
\begin{lemma}
\label{lem:fair}
Let $P\subseteq\mathbb{R}^d$ be a set of $n$ points, $b:P\to\mathbb{R}$, $\lambda\in\mathbb{R}$, and let $\Psi_{Fair}$ denote the Fair loss function.
Let $P'=\{\p\circ b(\p)\,|\,\p\in P\}$, where $\circ$ denotes vertical concatenation.
Let $C'$ be the output of a call to $L_\infty-\coreset(P',d)$ and let $C\subseteq P$ so that $C'=\{\q\circ b(\q)\,|\,\q\in C\}$.
Then for every $\w\in\mathbb{R}^d$, $\max_{\p\in P}\Psi_{Fair} \left(|\p^\top\w-b(\p)|\right)\le 8(d+1)^3\cdot\max_{\q\in C}\Psi_{Fair}\left(|\q^\top\w-b(\q)|\right).$
\end{lemma}
\begin{proof}
Since the claim is trivially true for $\w=0^d$, then it suffices to consider nonzero $\w\in\mathbb{R}^d$. 
Let $Y\in\calH_{d-1}$ such that $\w^\top\Y=0^{d-1}$ and $\Y^\top\w=0^d$.
For each $\p\in P$, let $\p'=\p\circ b(\p)=\begin{bmatrix}\p\\ b(\p)\end{bmatrix}$ denote the vertical concatenation of $\p$ with $b(\p)$.
We also define the vertical concatenation $\w'=\w\circ(-1)=\begin{bmatrix}\w\\ -1\end{bmatrix}$.
By setting $C$ to be the output of $\coreset$ on $P'=\{\p'\,|\,\p\in P\}$, then by Theorem~\ref{thm:main:apps},
\begin{align*}
\max_{\p\in P}\dist(\p',H(\w',0^{d+1}))^z&\le2^{z+1}(d+1)^{1.5z}\max_{\q\in C}\dist(\q',H(\w',0^{d+1}))^z.
\end{align*}
Thus for $z=2$, we have for every $\p\in P$,
\[|(\p')^\top\w'|^2\le8(d+1)^3\max_{\q\in C}|(\q')^\top\w'|^2.\]
The Fair loss function is monotonically increasing, so that $\Psi_{Fair}(|\p^\top\x-b(\p))$ increases as $|\p^\top\x-b(\p)|$ increases.
Therefore,
\begin{align*}
\max_{\p\in P}\Psi_{Fair}\left(|\p^\top\w-b(\p)|\right)\le\max_{\q\in C}\Psi_{Fair}\left(\sqrt{8}(d+1)^{1.5}|\q^\top\w-b(\q)|\right).
\end{align*}
The Fair loss function is not concave, so we cannot directly apply Lemma~\ref{lem:concave:struct}. 
However, if we define the function $f(x):=\lambda\sqrt{|x|}-\lambda^2\ln\left(1+\frac{\sqrt{|x|}}{\lambda}\right)$, then we have
\[\frac{d^2f}{dx^2}=-\frac{\lambda}{4\sqrt{x}(\lambda+\sqrt{x})^2},\]
which is non-positive for all $x\ge 0$. 
Thus by Lemma~\ref{lem:concave:struct}, we have for all $0\le x\le y$ that $\frac{f(x)}{x}\ge\frac{f(y)}{y}$. 
Since $f(x^2)=\Psi_{Fair}(x)$, then we have for all $0\le x\le y$ that $\frac{\Psi_{Fair}(x)}{x^2}\ge\frac{\Psi_{Fair}(y)}{y^2}$. 
Thus,
\begin{align*}
\max_{\p\in P}\Psi_{Fair}\left(|\p^\top\w-b(\p)|\right)\le8(d+1)^3\max_{\q\in C}\Psi_{Fair}\left(|\q^\top\w-b(\q)|\right).
\end{align*}
\end{proof}

\section{ADDITIONAL EXPERIMENTS}
\label{sec:exp_ext}
In this section, we carry additional experimental results evaluating our coreset against uniform sampling on real-world datasets, with respect to the projective clustering problem and its variants.

\begin{table*}[!htb]
\caption{\textbf{Summary of our results: } Our coreset construction was applied on various application of projective clustering, of which were robust regression as well as robust subspace clustering}
\centering
\begin{tabular}{|c|c|c|c|c|c|}
\hline
Problem type & Loss function & $k$ & $j$ &  Dataset & Figure \\
\hline
Robust $(2,2)$-projective clustering & $L_1-L_2$ & $2$ & $2$ & \ref{dataset:3} & \ref{fig:proj_CASP_2_2_l12} \\ \hline
Robust $(2,2)$-projective clustering & Huber & $2$ & $3$ & \ref{dataset:3} & \ref{fig:proj_CASP_2_3_huber}\\
\hline
\end{tabular}
\label{tab:more:summary_results}
\end{table*}

\begin{figure*}[tbh!]
\centering
\begin{subfigure}[t]{1\textwidth}
\centering
\includegraphics[width=.32\textwidth]{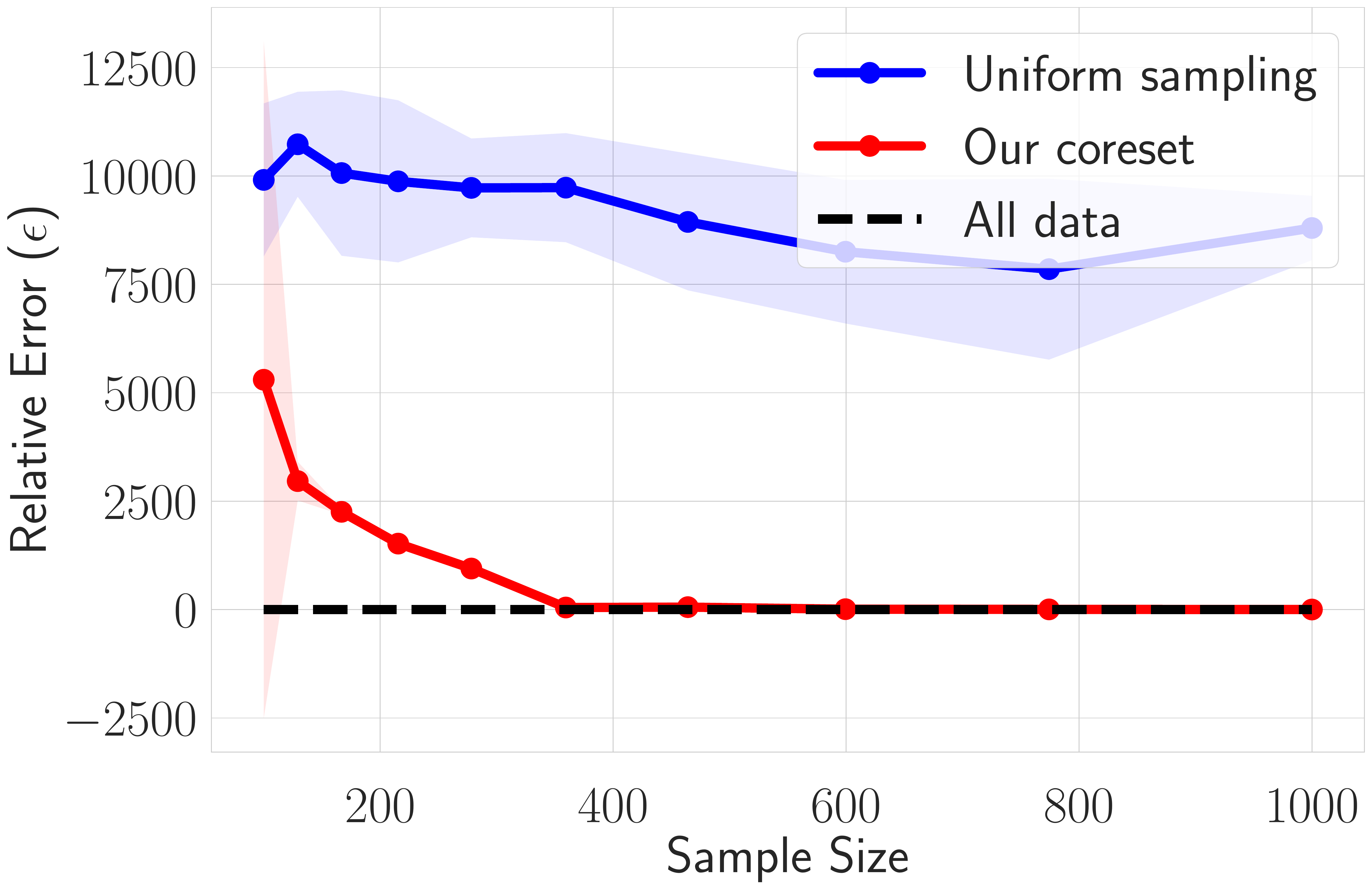}
\includegraphics[width=.32\textwidth]{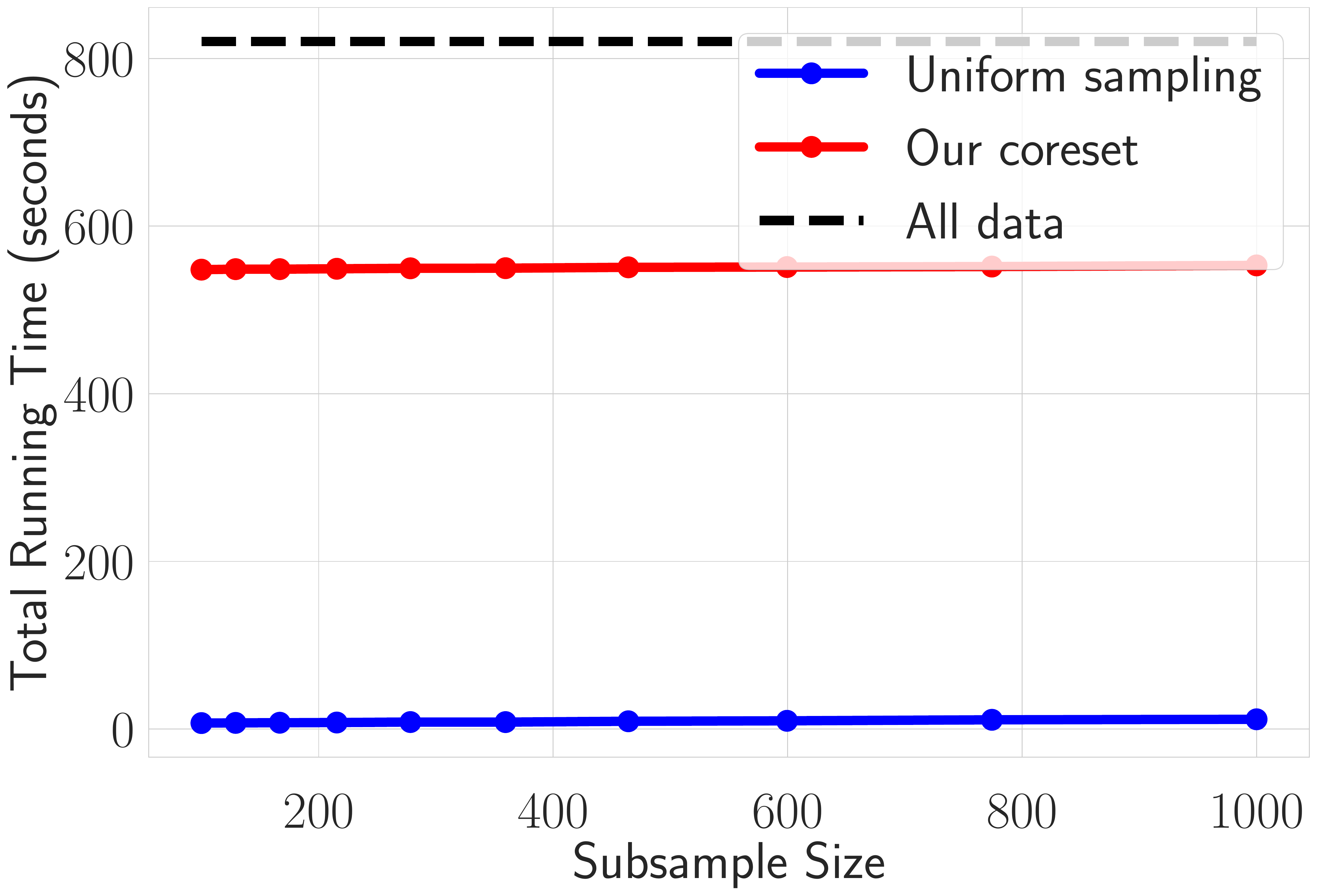}
\includegraphics[width=.32\textwidth]{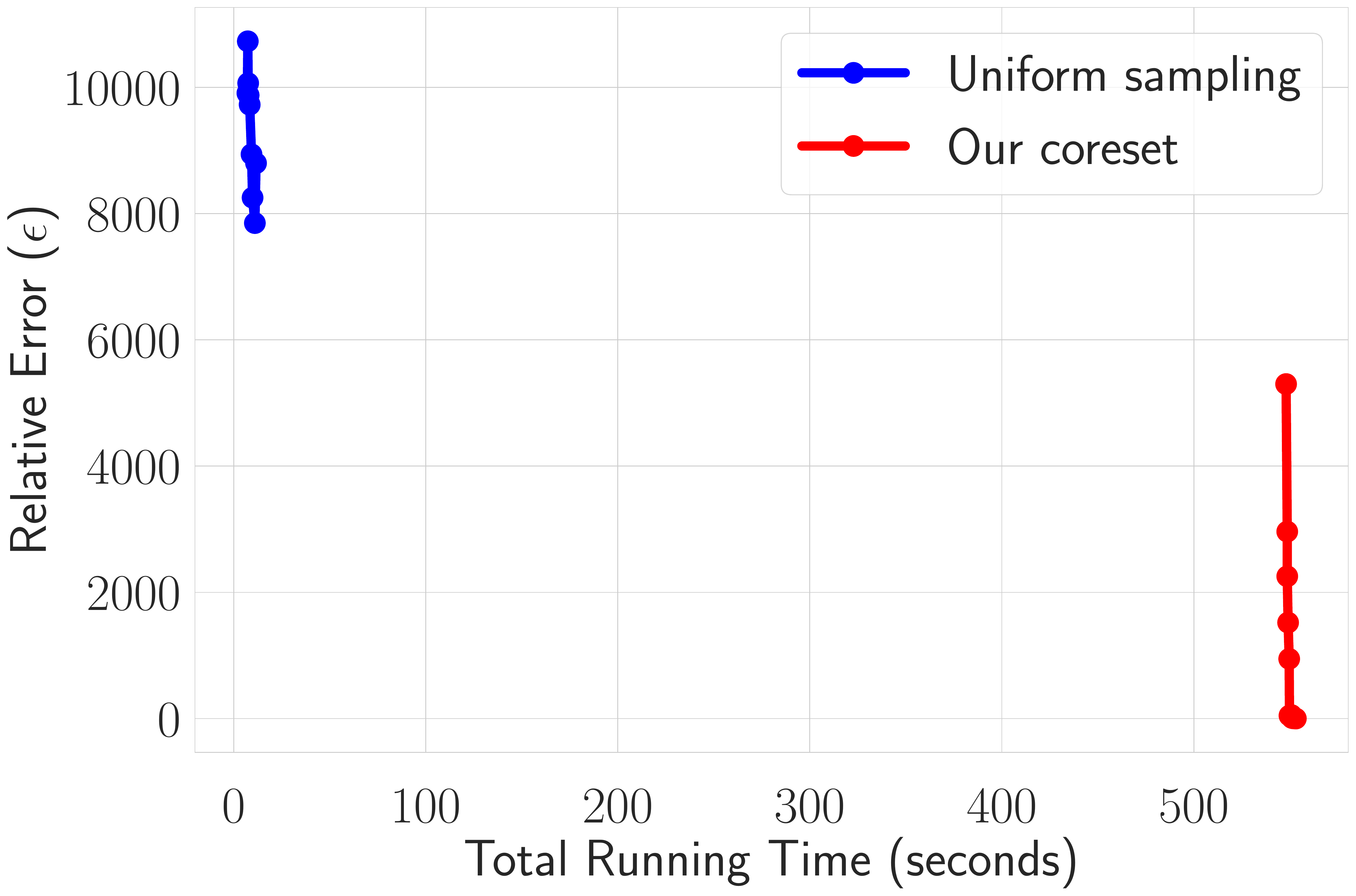}
\caption{}
\label{fig:proj_CASP_2_2_l12}
\end{subfigure}
\begin{subfigure}[t]{1\textwidth}
\centering
\includegraphics[width=.32\textwidth]{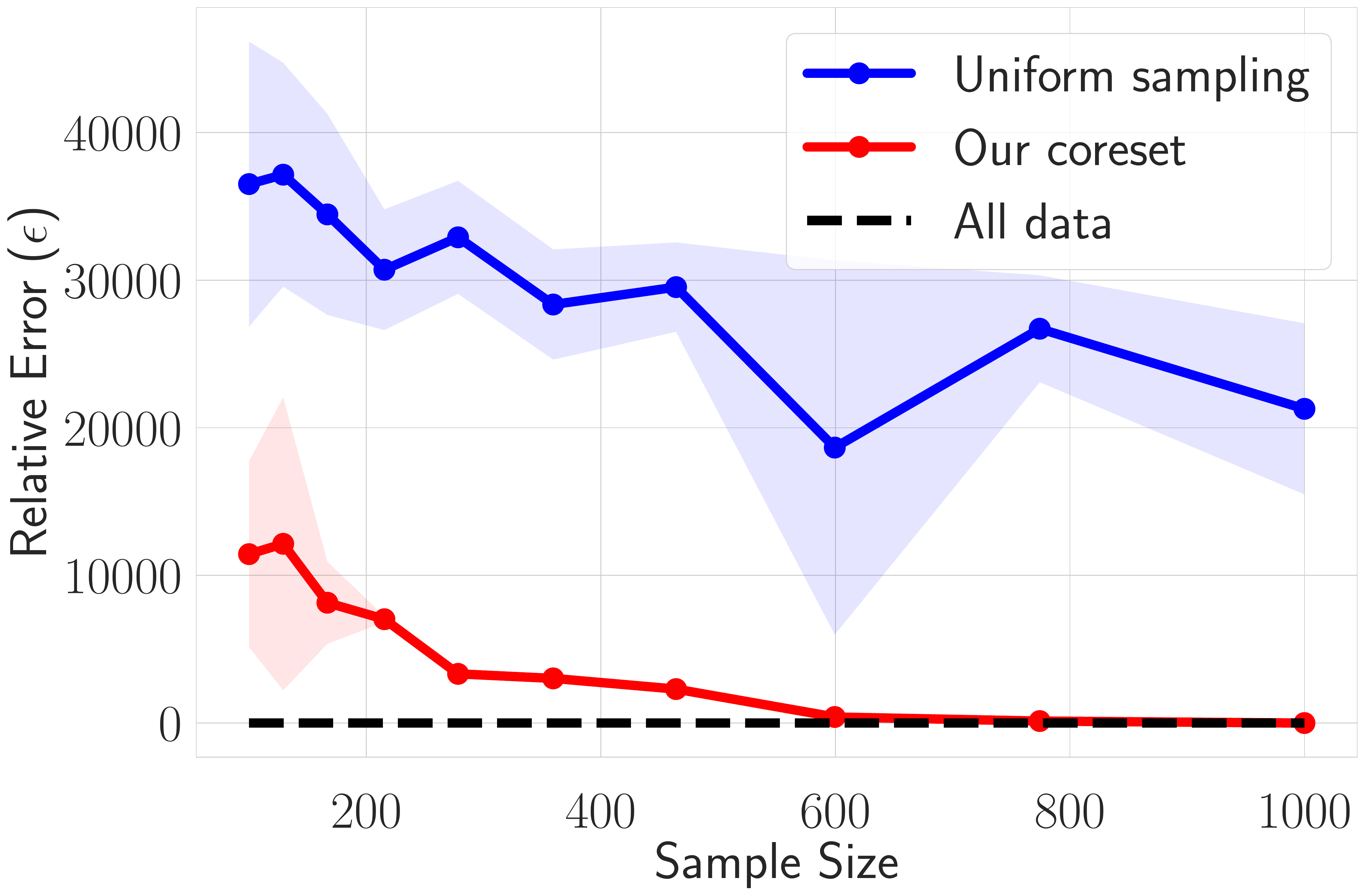}
\includegraphics[width=.32\textwidth]{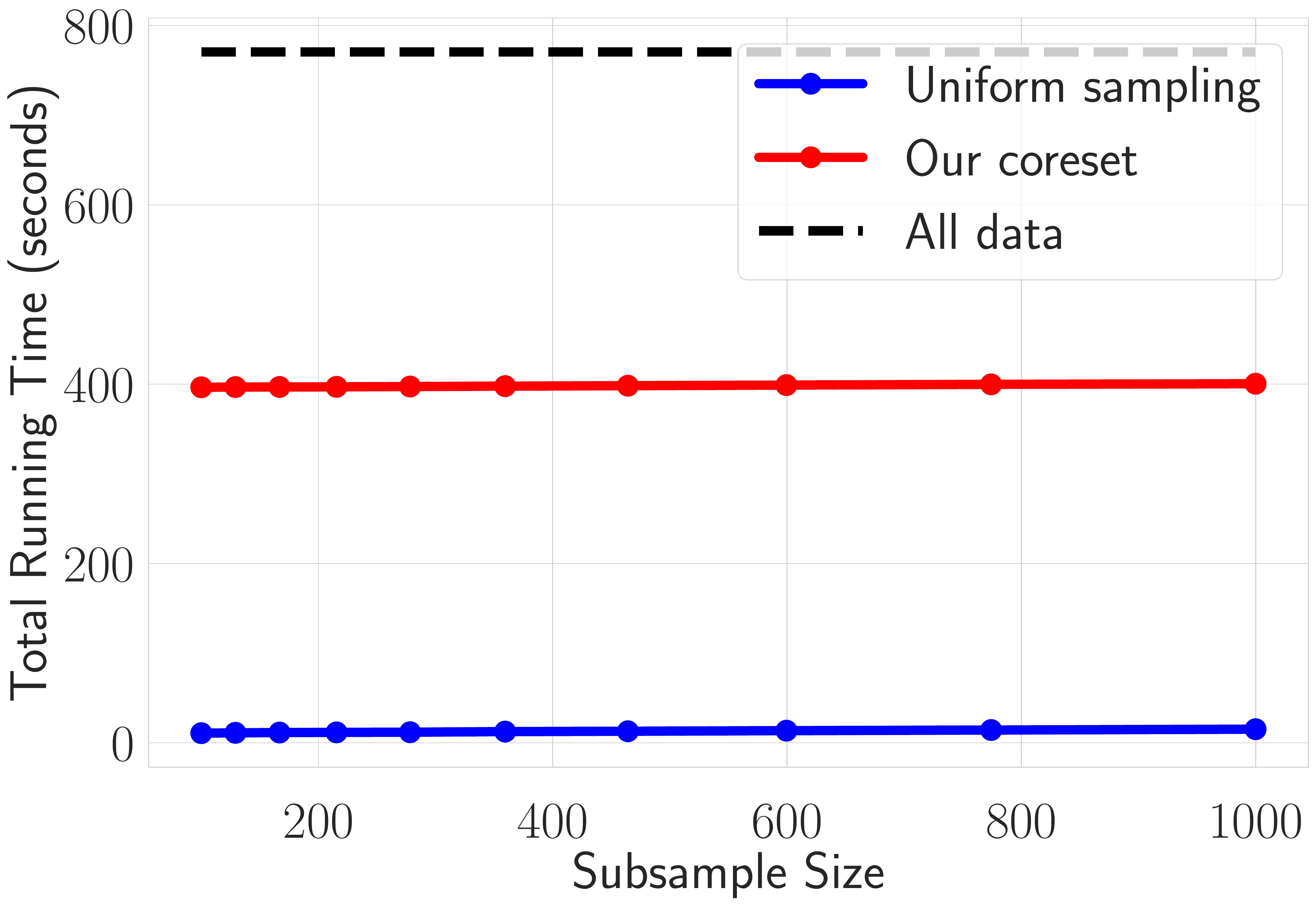}
\includegraphics[width=.32\textwidth]{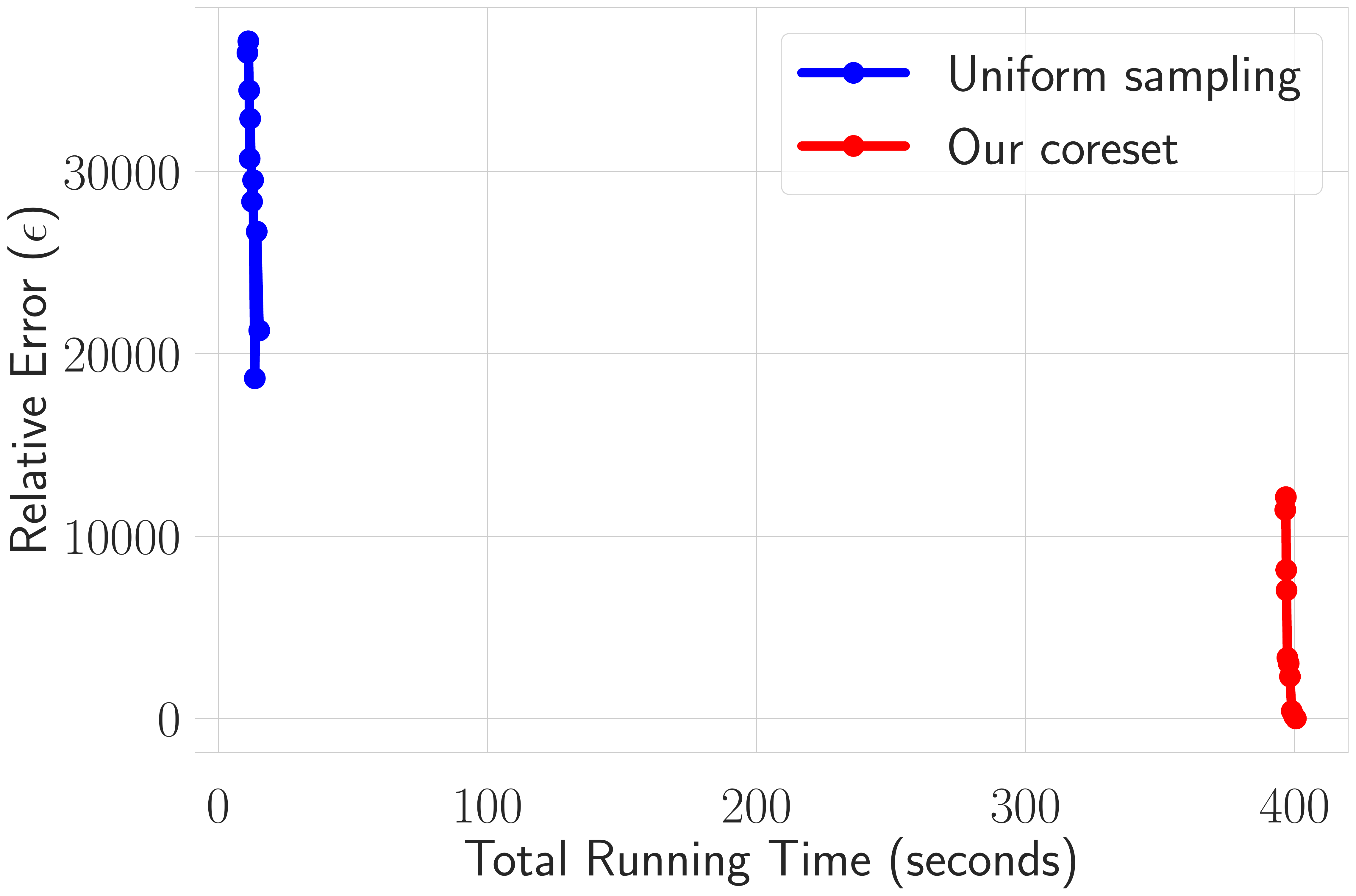}
\caption{}
\label{fig:proj_CASP_2_3_huber}
\end{subfigure}
\caption{Our experimental results: evaluating the efficacy of our coreset against uniform sampling.}
\label{fig:more:results}
\end{figure*}

\end{document}

%% file: preamble.tex
\usepackage{booktabs}       
\usepackage{amsfonts}       
\usepackage{nicefrac}       
\usepackage{microtype}      
\usepackage{xcolor}         
\usepackage{fullpage}
\usepackage{natbib}
\usepackage[backref=page]{hyperref}
\usepackage{algorithm}
\usepackage[noend]{algpseudocode}
\usepackage{amsmath,amssymb,amsfonts,amsthm}

\allowdisplaybreaks
\usepackage{subcaption}
\usepackage{graphicx}
\usepackage{color}
\usepackage{wrapfig}
\usepackage{tikz}
\usetikzlibrary{decorations.pathreplacing}
\usepackage{setspace}
\usepackage[framemethod=tikz]{mdframed}
\usepackage{framed}
\usepackage{subcaption}
\usepackage{thmtools}
\usepackage{thm-restate}
\usepackage[inline]{enumitem}

\newtheorem{theorem}{Theorem}[section]

\newtheorem{lemma}[theorem]{Lemma}

\newtheorem{definition}[theorem]{Definition}
\newtheorem{remark}[theorem]{Remark}

\newenvironment{proofof}[1]{\begin{trivlist} \item {\bf Proof
#1:~~}}
  {\qed\end{trivlist}}

\newcommand{\namedref}[2]{\hyperref[#2]{#1~\ref*{#2}}}

\newcommand{\br}[1]{\left\lbrace #1 \right\rbrace}

\renewcommand{\O}[1]{\ensuremath{\mathcal{O}\left(#1\right)}}

\newcommand{\eps}{\varepsilon}

\def \a    {\mdef{\mathbf{a}}}
\def \A    {\mdef{\mathbf{A}}}

\def \calF    {\mdef{\mathcal{F}}}
\def \G    {\mdef{\mathbf{G}}}

\def \calH    {\mdef{\mathcal{H}}}
\def \I    {\mdef{\mathbb{I}}}

\def \calM    {\mdef{\mathcal{M}}}

\def \calQ    {\mdef{\mathcal{Q}}}

\def \calV    {\mdef{\mathcal{V}}}

\def \X    {\mdef{\mathbf{X}}}
\def \Y    {\mdef{\mathbf{Y}}}

\def \W    {\mdef{\mathbf{W}}}

\def \c    {\mdef{\mathbf{c}}}

\def \p    {\mdef{\mathbf{p}}}
\def \q    {\mdef{\mathbf{q}}}

\def \s    {\mdef{\mathbf{s}}}
\def \u    {\mdef{\mathbf{u}}}
\def \w    {\mdef{\mathbf{w}}}
\def \v    {\mdef{\mathbf{v}}}
\def \x    {\mdef{\mathbf{x}}}

\def \coreset    {\mdef{\textsc{Coreset}}}

\newcommand{\mdef}[1]{{\ensuremath{#1}}}  

\DeclareMathOperator*{\polylog}{polylog}
\DeclareMathOperator*{\poly}{poly}
\DeclareMathOperator*{\conv}{conv}
\DeclareMathOperator*{\dist}{dist}
\DeclareMathOperator*{\Span}{span}

\newcommand{\abs}[1]{\mdef{\left|#1\right|}}         

\newcommand{\ignore}[1]{}

\newif\ifnotes\notestrue 
\ifnotes
\newcommand{\samson}[1]{\textcolor{purple}{{\bf (Samson:} {#1}{\bf ) }} \marginpar{\tiny\bf
             \begin{minipage}[t]{0.5in}
               \raggedright S:
            \end{minipage}}}            							
\else
\newcommand{\samson}[1]{}
\fi

\makeatletter
\renewcommand*{\@fnsymbol}[1]{\textcolor{blue}{\ensuremath{\ifcase#1\or *\or \dagger\or \ddagger\or
 \mathsection\or \triangledown\or \mathparagraph\or \|\or **\or \dagger\dagger
   \or \ddagger\ddagger \else\@ctrerr\fi}}}
\makeatother

\providecommand{\email}[1]{\href{mailto:#1}{\nolinkurl{#1}\xspace}}
\newcommand{\say}[1]{``#1''}

\definecolor{mahogany}{rgb}{0.75, 0.25, 0.0}
\definecolor{darkblue}{rgb}{0.0, 0.0, 0.55}
\definecolor{darkpastelgreen}{rgb}{0.01, 0.75, 0.24}
\definecolor{bleudefrance}{rgb}{0.19, 0.55, 0.91}
\definecolor{richmaroon}{rgb}{0.69, 0.19, 0.38}
\definecolor{darkgreen}{rgb}{0.0, 0.2, 0.13}
\definecolor{darkgoldenrod}{rgb}{0.72, 0.53, 0.04}
\definecolor{darkred}{rgb}{0.55, 0.0, 0.0}
\hypersetup{
     colorlinks   = true,
     citecolor    = mahogany,
		 linkcolor		= olive
}